\documentclass{article}

\PassOptionsToPackage{numbers}{natbib}

\usepackage[preprint]{neurips_2021}

\usepackage[utf8]{inputenc} %
\usepackage[T1]{fontenc}    %
\usepackage{hyperref}       %
\usepackage{url}            %
\usepackage{booktabs}       %
\usepackage{amsfonts}       %
\usepackage{nicefrac}       %
\usepackage{microtype}      %
\usepackage{xcolor}         %

\usepackage[ruled,vlined]{algorithm2e}
\usepackage{amsmath}
\usepackage{amsthm}
\usepackage{dsfont}
\usepackage{graphicx}
\usepackage{caption}
\usepackage{subcaption}
\usepackage{multirow}
\usepackage{xcolor}
\usepackage{diagbox}
\usepackage{enumitem}

\title{On UMAP's true loss function}

\author{%
  Sebastian Damrich \hspace{0.5cm}   Fred A. Hamprecht \\
  HCI/IWR\\
  Heidelberg University\\
   69117 Heidelberg, Germany \\
  \texttt{\{sebastian.damrich, fred.hamprecht\}@iwr.uni-heidelberg.de}
}

\begin{document}

\maketitle
\newtheorem{theorem}{Theorem}[section]
\begin{abstract}
UMAP has supplanted $t$-SNE as state-of-the-art for visualizing high-dimensional datasets in many disciplines, but the reason for its success is not well understood. In this work, we investigate UMAP's sampling based optimization scheme in detail. We derive UMAP's effective loss function in closed form and find that it differs from the published one. As a consequence, we show that UMAP does not aim to reproduce its theoretically motivated high-dimensional UMAP similarities. Instead, it tries to reproduce  similarities that only encode the shared $k$ nearest neighbor graph, thereby challenging the previous understanding of UMAP's effectiveness. Instead, we claim that the key to UMAP's success is its implicit balancing of attraction and repulsion resulting from negative sampling. This balancing in turn facilitates optimization via gradient descent. We corroborate our theoretical findings on toy and single cell RNA sequencing data.
\end{abstract}

\section{Introduction}
Today's most prominent methods for non-parametric, non-linear dimension reduction are \mbox{$t$-Distributed} Stochastic Neighbor Embedding ($t$-SNE) \cite{van2008visualizing, van2014accelerating} and Uniform Manifold Approximation and Projection for Dimension Reduction (UMAP)~\cite{mcinnes2018umap}. The heart of UMAP is claimed to be its sophisticated method for extracting the high-dimensional similarities, motivated in the language of algebraic topology and category theory. However, the reason for UMAP's excellent visualizations is not immediately obvious from this approach. In particular, UMAP's eponymous uniformity assumption is arguably difficult to defend for the wide variety of datasets on which UMAP performs well. Therefore, it is not well understood what exactly about UMAP is responsible for its great visualizations. 

Both $t$-SNE and UMAP have to overcome the computational obstacle of considering the quadratic number of interactions between all pairs of points. The breakthrough for $t$-SNE came with a Barnes-Hut approximation~\cite{van2014accelerating}. Instead, UMAP employs a sampling based approach to avoid a quadratic number of repulsive interactions. Other than~\cite{bohm2020unifying} little attention has been paid to this sampling based optimization scheme. In this work, we fill this gap and analyze UMAP's optimization method in detail. In particular, we derive the effective, closed form loss function which is truly minimized by UMAP's optimization scheme. While UMAP's use of negative sampling was intended to avoid quadratic complexity, we find, surprisingly, that the resulting effective loss function differs significantly from UMAP's purported loss function. The weight of the loss function's repulsive term is drastically reduced. As a consequence, UMAP is not actually geared towards reproducing the clever high-dimensional similarities. In fact, we show that most information beyond the shared $k$NN graph connectivity is essentially ignored as UMAP actually approximates a binarized version of the high-dimensional similarities. These theoretical findings underpin some empirical observations in~\cite{bohm2020unifying} and demonstrate that the gist of UMAP is not in its high-dimensional similarities. This resolves the disconnect between UMAP's uniformity assumption and its success on datasets of varying density. From a user's perspective it is important to gain an intuition for deciding which features of a visualization can be attributed to the data and which ones are more likely artifacts of the visualization method. With our analysis, we can explain UMAP's tendency to produce crisp, locally one-dimensional substructures as a side effect of its optimization.

Without the motivation of reproducing sophisticated high-dimensional similarities in embedding space, it seems unclear why UMAP performs well. We propose an alternative explanation for UMAP's success: The sampling based optimization scheme balances the attractive and repulsive loss terms despite the sparse high-dimensional attraction. Consequently, UMAP can leverage the connectivity information of the shared $k$NN graph via gradient descent effectively. 

\section{Related Work}\label{sec:related_work}

For most of the past decade $t$-SNE~\cite{van2008visualizing, van2014accelerating} was considered state-of-the-art for non-linear dimension reduction. In the last years 
UMAP~\cite{mcinnes2018umap} at least ties with $t$-SNE. In both cases, points are embedded so as to reproduce high-dimensional similarities; but the latter are sparse for UMAP and do not need to be normalized over the entire dataset. Additionally, $t$-SNE adapts the local scale of high-dimensional similarities by achieving a predefined perplexity, while UMAP uses its uniformity assumption. The low-dimensional similarity functions also differ. Recently, B\"{o}hm et al.~\cite{bohm2020unifying} placed both UMAP and $t$-SNE on a spectrum of dimension reduction methods that mainly differ in the amount of repulsion employed. They argue that UMAP uses less repulsion than $t$-SNE. A parametric version of UMAP was proposed in~\cite{sainburg2020parametric}.

UMAP's success, in particular in the biological community~\cite{becht2019dimensionality, packer2019lineage}, sparked interest in understanding UMAP more deeply. The original paper~\cite{mcinnes2018umap} motivates the choice of the high-dimensional similarities using concepts from algebraic topology and category theory and thus justifies UMAP's transition from local similarities $\mu_{i \to j}$ to global similarities $\mu_{ij}$. The authors find that while the algorithm focuses on reproducing the local similarity pattern similar to $t$-SNE, it achieves better global results than $t$-SNE. In contrast, \citet{kobak2021initialization} attribute the better global properties of UMAP visualizations to the more informative initialization and show that $t$-SNE manages to capture more global structure if initialized in a similar way.

\citet{narayan2021assessing} observe that UMAP's uniformity assumption leads to visualizations in which denser regions are more spread out while sparser regions get overly contracted. They propose an additional loss term that aims to reproduce the local density around each point and thus spaces sparser regions out. We provide an additional explanation for overly contracted visualizations: UMAP does not reproduce the high-dimensional similarities but exaggerates the attractive forces over the repulsive ones, which can result in overly crisp visualizations, see Figures~\ref{subfig:toy_ring_UMAP} and \ref{subfig:c_elegans_UMAP}.

Our work aligns with \citet{bohm2020unifying}. The authors conjecture that the sampling based optimization procedure of UMAP prevents the minimization of the supposed loss function, thus not reproducing the high-dimensional similarities in embedding space. They substantiate this hypothesis by qualitatively estimating the relative size of attractive and repulsive forces. In addition, they implement a Barnes-Hut approximation to the loss function~\eqref{eq:UMAP_obj_embd} and find that it yields a diverged embedding. We analyze UMAP's sampling procedure more closely and compute UMAP's true loss function in closed form and contrast it against the exact supposed loss in Section~\ref{sec:eff_loss}. Based on this analytic effective loss function, we can further explain \citet{bohm2020unifying}'s empirical finding that the specific high-dimensional similarities provide little gain over the binary weights of a shared $k$NN graph,\footnote{The shared $k$ nearest neighbor graph contains an edge $ij$ if $i$ is among $j$'s $k$ nearest neighbors or vice versa.} see Section~\ref{sec:target_sims}. Finally, our theoretical framework leads us to a new tentative explanation for UMAP's success discussed in Section~\ref{sec:discussion}.

\section{Background: UMAP}\label{sec:UMAP_overview}
The key idea of UMAP~\cite{mcinnes2018umap} is to compute pairwise similarities in high-dimensional space which inform the optimization of the low-dimensional embedding. Let $x_1, ..., x_n \in \mathbb{R}^D$ be high-dimensional, mutually distinct data points for which low-dimensional embeddings $e_1, ..., e_n \in \mathbb{R}^d$ shall be found, where $d\ll D$, often $d=2$ or $3$. First, UMAP computes high-dimensional similarities between the data points. To do so, the $k$ nearest neighbor ($k$NN) graph is computed, so that $i_1, ..., i_k$ denote the indices of $x_i$'s $k$ nearest neighbors in increasing order of distance to $x_i$.  Then, using its uniformity assumption, UMAP fits a local notion of similarity for each data point $i$ by selecting a scale $\sigma_i$ such that the total similarity of each point to its $k$ nearest neighbors is normalized, i.e. find $\sigma_i$ such that
\begin{equation}
    \sum_{\kappa=1}^k \exp\left(-(d(x_i, x_{i_\kappa}) - d(x_i, x_{i_1})) / \sigma_i\right) = \log_2(k).
\end{equation}
This defines the directed high-dimensional similarities
\begin{equation}
    \mu_{i\to j} = 
\begin{cases}
    \exp\left(-(d(x_i, x_{j}) - d(x_i, x_{i_1})) / \sigma_i\right) \text{ for } j \in \{i_1, \dots, i_k\}\\
    0 \text{ else.}
\end{cases}
\end{equation}
Finally, these are symmetrized to obtain undirected high-dimensional similarities or input similarities between items $i$ and $j$
\begin{equation}\mu_{ij} = \mu_{i\to j} + \mu_{j\to i} - \mu_{i\to j}\mu_{j\to i} \in [0, 1].
\end{equation}
While each node has exactly $k$ non-zero directed similarities $\mu_{i\to j}$ to other nodes which sum to $\log_2(k)$, this does not hold exactly after symmetrization. Nevertheless, typically the $\mu_{ij}$ are highly sparse, each node has positive similarity to on about $k$ other nodes and the degree of each node \mbox{$d_i = \sum_{j=1}^n \mu_{ij}$} is approximately constant and close to $\log_2(k)$, see Figures~\ref{fig:degree_hist} and~\ref{fig:degree_hist_kNN}. For convenience of notation, we set $\mu_{ii} = 0$ and define the total similarity as $\mu_\text{tot} =\frac{1}{2} \sum_{i=1}^n d_i$.

Distance in embedding space is transformed to low-dimensional similarity by a smooth approximation to the high-dimensional similarity function, $\phi(d; a,b) = (1+ad^{2b})^{-1}$, using the same slack and scale for all points. The shape defining parameters $a, b$ are essentially hyperparameters of UMAP. We will overload notation and write 
\begin{equation}
    \nu_{ij} = \phi(||e_i - e_j||) = \phi(e_i, e_j)
\end{equation}
for the low-dimensional similarities or embedding similarities and usually suppress their dependence on $a$ and $b$. 

With this setup, UMAP supposedly optimizes the following objective function with respect to the embeddings $\{e_1, \dots, e_n\}$ approximately:
\begin{alignat}{2}
    \mathcal{L}(\{e_i\}| \{\mu_{ij}\}) 
    &= - 2\sum_{1\leq i < j \leq n} \mu_{ij} \log(\nu_{ij}) &+ &(1-\mu_{ij}) \log(1-\nu_{ij}) \label{eq:UMAP_obj_sim}\\
    &= - 2\sum_{1\leq i < j \leq n} \mu_{ij}\underbrace{ \log(\phi(e_i, e_j))}_{-\mathcal{L}^a_{ij}} &+ & (1-\mu_{ij}) \underbrace{\log(1-\phi(e_i, e_j))}_{-\mathcal{L}^r_{ij}}.
    \label{eq:UMAP_obj_embd}
\end{alignat}
While the high-dimensional similarities $\mu_{ij}$ are symmetric, UMAP's implementation does consider their direction during optimization. For this reason, our losses in equations~\eqref{eq:UMAP_obj_sim} and~\eqref{eq:UMAP_obj_embd} differ by a factor of $2$ from the one given in~\cite{mcinnes2018umap}.
Viewed through the lens of a force-directed model, the derivative of the first term in each summand of the loss function, $-\partial \mathcal{L}^a_{ij} / \partial e_i$, captures the attraction of $e_i$ to $e_j$ due to the high-dimensional similarity $\mu_{ij}$ and the derivative of the second term, $-\partial \mathcal{L}^r_{ij} / \partial e_i$, represents the repulsion that $e_j$ exerts on $e_i$ due to a lack of similarity in high dimension, $1-\mu_{ij}$.
Alternatively, the loss can be seen as the sum of binary cross entropy losses for each pairwise similarity. Thus, it is minimized if the low-dimensional similarities $\nu_{ij}$ exactly match their high-dimensional counterparts $\mu_{ij}$, that is, if UMAP manages to perfectly reproduce the high-dimensional similarities in low-dimensional space. 

UMAP uses a sampling based stochastic gradient descent to optimize its low-dimensional embedding typically starting from a Laplacian Eigenmap initialization~\cite{belkin2001laplacian, kobak2021initialization}. The main contribution of this paper is to show that the sampling based optimization in fact leads to a different objective function, so that UMAP does not reproduce the high-dimensional similarities in low-dimensional space, see Sections~\ref{sec:UMAP_not_reprod} to \ref{sec:target_sims}.
\begin{figure}[tb]
    \begin{subfigure}[t]{0.33\textwidth}
        \centering
        \includegraphics[width=.9\linewidth]{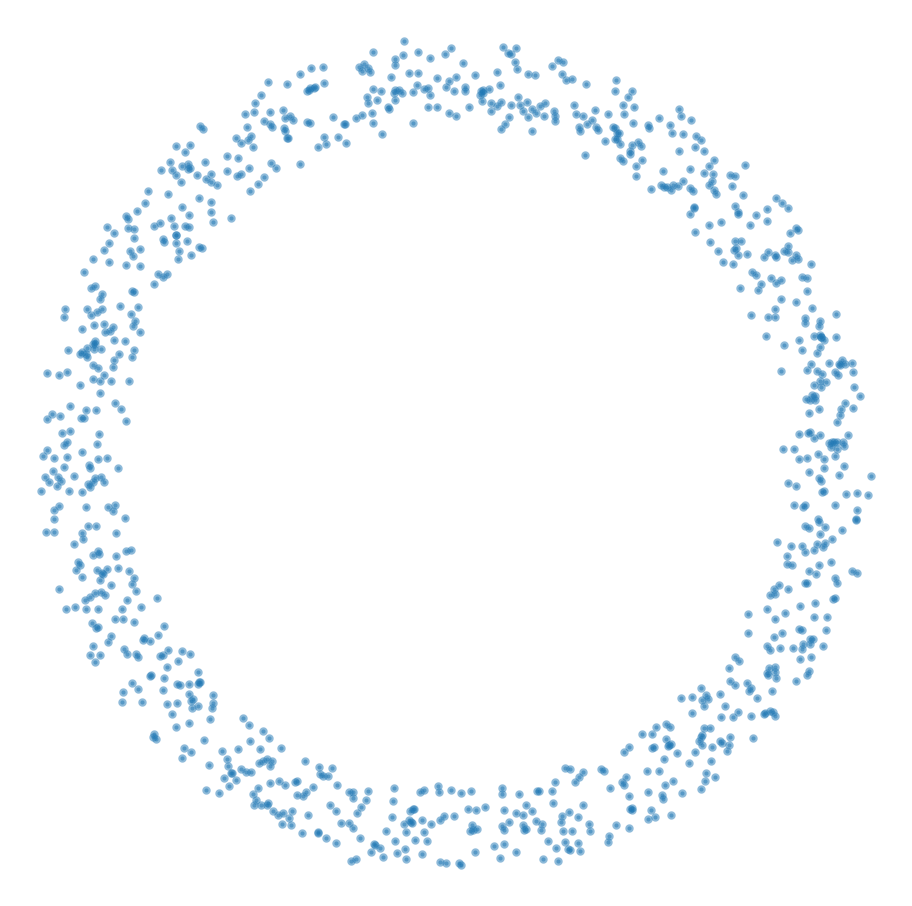}
        \caption{Original data}
        \label{subfig:toy_ring_original}
    \end{subfigure}%
    \begin{subfigure}[t]{0.33\textwidth}
        \centering
        \includegraphics[width=.9\linewidth]{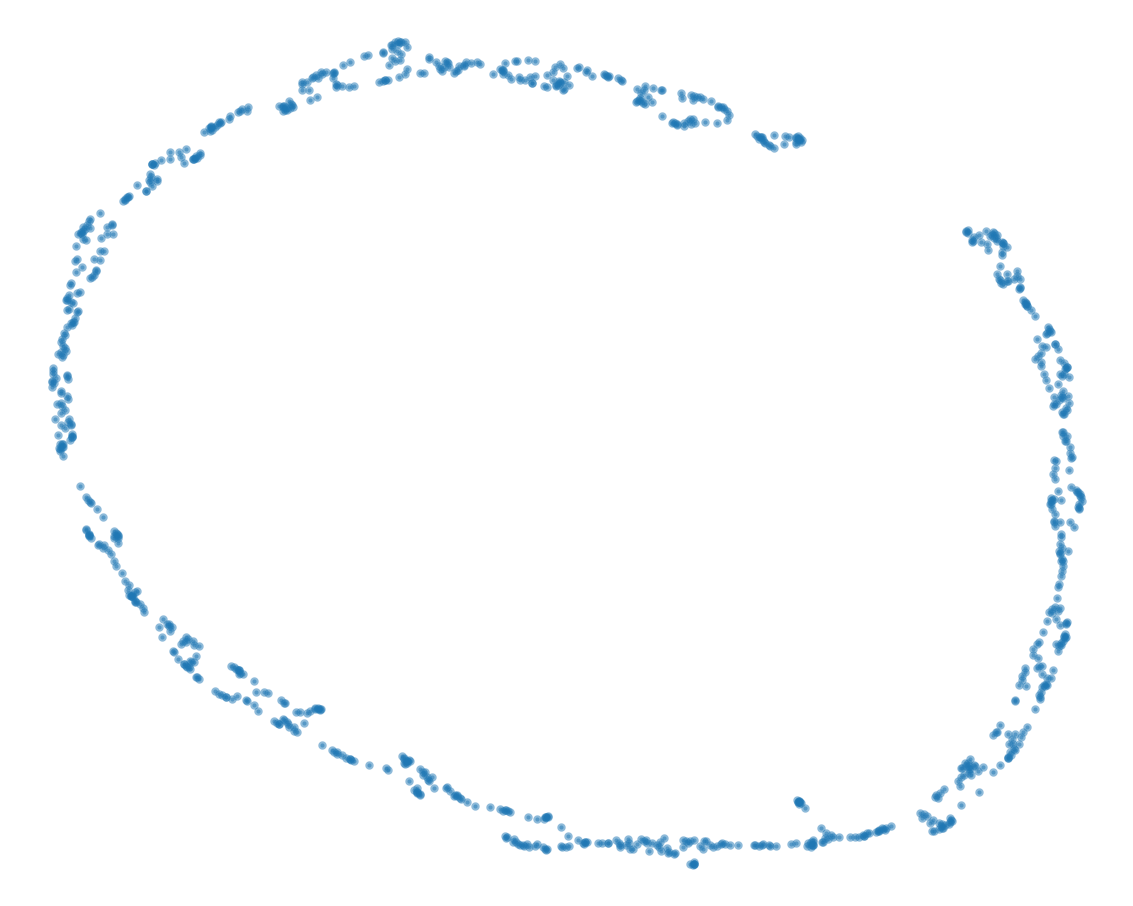}
        \caption{UMAP}
        \label{subfig:toy_ring_UMAP}
    \end{subfigure}%
    \begin{subfigure}[t]{0.33\textwidth}
        \centering
        \includegraphics[width=.9\linewidth]{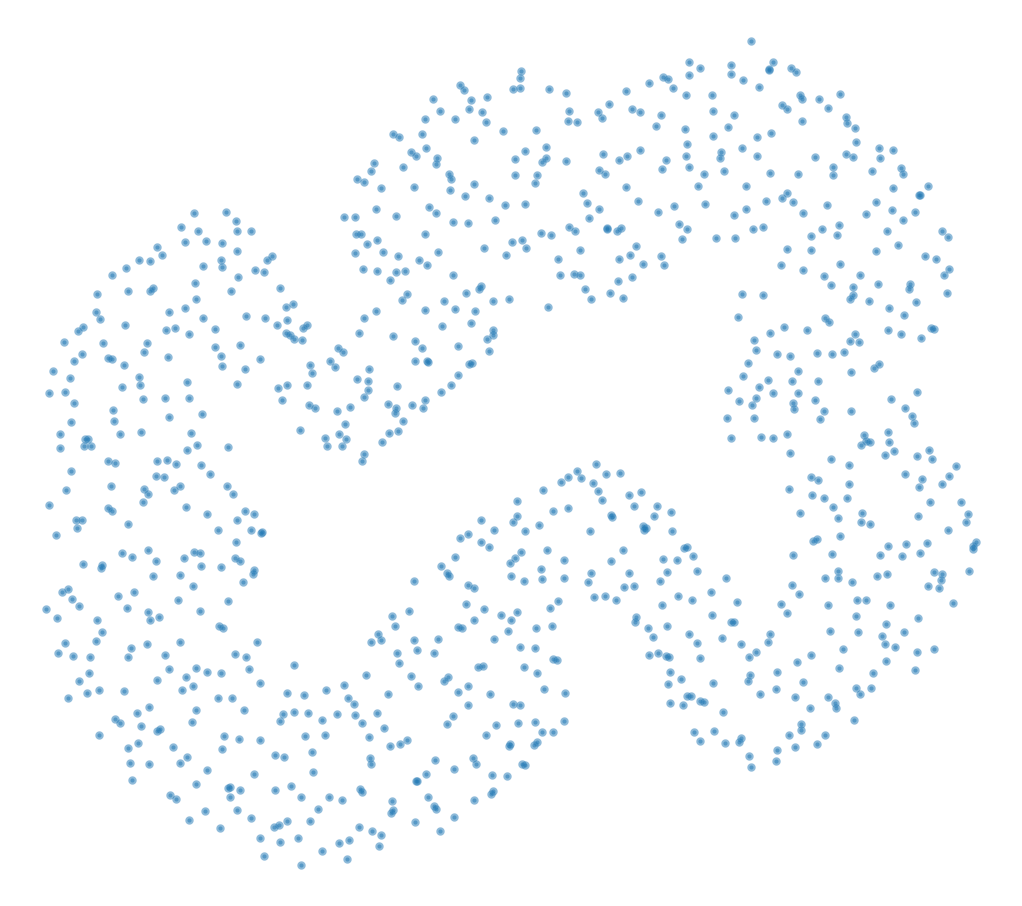}
        \caption{UMAP from dense similarities}
        \label{subfig:toy_ring_dense}
    \end{subfigure}
    \caption{UMAP does not preserve the data even when no dimension reduction is required. \ref{subfig:toy_ring_original}~Original data consisting of 1000 points sampled uniformly from a ring in 2D. \ref{subfig:toy_ring_UMAP}~Result of UMAP after 10000 epochs, initialized with the original data. The circular shape is visible but the ring width is nearly completely contracted. \ref{subfig:toy_ring_dense}~Result of UMAP after 10000 epochs for dense input similarities computed from the original data with $\phi$, initialized at the original embedding. No change from the initialization would be optimal in this setting. Instead the output has spurious curves and larger width. Additional figures with the default number of epochs and initialization can be found in Figure~\ref{fig:add_toy_ring} in the Appendix.}
    \label{fig:toy_ring}
\end{figure}

\section{UMAP does not reproduce high-dimensional similarities}\label{sec:UMAP_not_reprod}
UMAP produces scientifically useful visualizations for several domains and is fairly robust to its hyperparameters. Since any visualization of intrinsically high-dimensional data must be somehow unfaithful, it is not straightforward to check the visualization quality other than by its downstream use. Some quantitative measures exist such as the Pearson correlation between high- and low-dimensional distances used e.g. in ~\cite{kobak2021initialization, becht2019dimensionality}. We follow a different route to show unexpected properties of UMAP. Consider the toy example of applying UMAP to data that is already low-dimensional, such that no reduction in dimension is required: $D=d=2$. Ideally, the data would be preserved in this situation. One might also expect that UMAP achieves its aim of perfectly reproducing the high-dimensional similarities. Surprisingly, neither of these expectations is met. In Figure~\ref{fig:toy_ring}, we depict $2$D UMAP visualizations of a two-dimensional uniform ring dataset. To ease UMAP's task, we initialized the embedding with the original data, hoping that UMAP's optimization would just deem this layout to be optimal. We also used a longer run time to ensure convergence of UMAP's optimization procedure. The results with the default number of optimization epochs and initialization are qualitatively similar and shown in Figure~\ref{fig:add_toy_ring} in the Appendix. UMAP manages to capture the ring shape of the data, but changes its appearance significantly. As observed on many real-world datasets, confer also Figure~\ref{subfig:c_elegans_UMAP}, UMAP contracts the width of the ring nearly to a line, see Figure~\ref{subfig:toy_ring_UMAP}. Whether this exaggeration of the ring shape is useful depends on the use case. Note that this finding goes beyond \citet{narayan2021assessing}'s observation that over-contraction happens in regions of low density since our toy dataset is sampled uniformly from a circular ring. 

As described in Section~\ref{sec:UMAP_overview}, UMAP employs different methods in input and embedding space to transform distances to similarities. In particular, in input space similarities are zero for all but the closest neighbors, while in embedding space they are computed with the heavy-tailed function $\phi$. To test whether this prevents the reproduction of the input data, we also computed the dense similarities on the original data with $\phi$ and used this as input similarities for the embedding in Figure~\ref{subfig:toy_ring_dense}. Since we also initialize with the original dataset, a global optimum of the objective function~\eqref{eq:UMAP_obj_embd} for this choice of input similarity, one would expect no change by UMAP's optimization scheme. However, we observe that in this setting, UMAP produces spurious curves and increases the width of the ring. 

\citet{bohm2020unifying} implemented a Barnes-Hut approximation of UMAP's objective function~\eqref{eq:UMAP_obj_embd}, which produced a diverged embedding. Inspired by this finding, we compute the loss values according to equation~\eqref{eq:UMAP_obj_sim} for various input and embedding similarities $\mu_{ij}$ and $\nu_{ij}$ in our toy example, see Table~\ref{tab:toy_ring}. %
Consider the row with the usual input similarities ($\mu_{ij} = \mu(\{x_1, \dots, x_n\})$). In a completely diverged embedding, all self similarities are one and all others zero ($\nu_{ij} = \mathds{1}(i == j)$). We find that the loss for such an embedding is lower than for the optimized UMAP embedding. This is in accordance with \citet{bohm2020unifying}'s Barnes-Hut experiment and shows that UMAP does not optimize its supposed objective function~\eqref{eq:UMAP_obj_embd} as a diverged embedding is approximately feasible in two dimensions. This discrepancy is not just due to the fact that input and embedding similarities are computed differently: The second row of Table~\ref{tab:toy_ring} contains loss values for the setting in which we use the dense similarities as input similarities, as in Figure~\ref{subfig:toy_ring_dense}. We initialize the embedding at the optimal loss value ($\nu_{ij} = \phi(\{x_1, \dots, x_n\}) = \mu_{ij}$), but UMAP's optimization moves away from this layout and towards an embedding with higher loss ($\nu_{ij} = \phi(\{e_1, \dots, e_n\})$) although we always compute similarity in the same manner. Clearly, UMAP's optimization yields unexpected results.

\begin{table}[tb]
  \caption{UMAP loss value for various combinations of input and embedding similarities, $\mu_{ij}$, $\nu_{ij}$, of the toy example in Figure~\ref{fig:toy_ring}. The loss for the UMAP embedding (middle column) is always higher than for another two-dimensional layout (\textbf{bold}). Hence, UMAP does not minimize its purported loss. Results are averaged over 7 runs, see also Appendix~\ref{subapp:stability}.}
  \label{tab:toy_ring}
  \centering
  \begin{tabular}{lccc}
    \toprule
    &\multicolumn{3}{c}{Embedding similarities $\nu_{ij}$}                   \\
    \cmidrule(r){2-4}
     &   $\mathds{1}(i == j)$ & $\phi(\{e_1, \dots, e_n\})$& $\phi(\{x_1, \dots, x_n\})$ \\
    Input similarities $\mu_{ij}$& (diverged layout) & (UMAP result) & (input layout) \\ 
    \midrule
    $\mu(\{x_1, \dots, x_n\})$& $\mathbf{62959 \pm 82} $ & $70235 \pm 1301$ & $136329 \pm 721$\\
    $\phi(\{x_1, \dots, x_n\}) $& $902757 \pm 2788$ & $331666\pm 1308$& $\mathbf{224584 \pm 8104}$ \\
    \bottomrule
  \end{tabular}
\end{table}

\section{UMAP's sampling strategy and effective loss function}\label{sec:eff_loss}
UMAP uses a sampling based approach to optimize its loss function, in order to reduce complexity. A simplified version of the sampling procedure can be found in Algorithm~\ref{alg:npUMAP_sample}.  Briefly put, an edge $ij$ is sampled according to its high-dimensional similarity and the embeddings $e_i$ and $e_j$ of the incident nodes are pulled towards each other. For each such sampled edge $ij$, the algorithm next samples $m$ negative samples $s$ uniformly from all nodes and the embedding of $i$ is repelled from that of each negative sample. Note that the embeddings of the negative samples are not repelled from that of $i$, see commented line~\ref{algline:push_tail}. So there are three types of gradient applied to an embedding $e_i$ during an epoch:
\begin{enumerate}[wide, labelindent=0pt]
    \item $e_i$ is pulled towards $e_j$ when edge $ij$ is sampled, see line~\ref{algline:pull_head}
    \item $e_i$ is pulled towards $e_j$ when edge $ji$ is sampled, see line~\ref{algline:pull_tail}
    \item $e_i$ is pushed away from embeddings of negative samples when some edge $ij$ is sampled, see line~\ref{algline:push_head}.
\end{enumerate}
The full gradient on embedding $e_i$ during epoch $t$ is, according to UMAP's implementation, given by
\begin{equation}
    g_i^t =  \sum_{j=1}^n X_{ij}^t \cdot \frac{\partial \mathcal{L}^a_{ij}}{\partial e_i}   +X_{ji}^t \cdot \frac{\partial \mathcal{L}^a_{ji}}{\partial e_i}  + X_{ij}^t \cdot \sum_{s=1}^n Y^t_{ij, s} \cdot \frac{\partial \mathcal{L}^r_{is}}{\partial e_i},
\end{equation}
where $X_{ab}^t$ is the binary random variable indicating whether edge $ab$ was sampled in epoch $t$ and $Y^t_{ab, s}$ is the random variable for the number of times $s$ was sampled as negative sample for edge $ab$ in epoch $t$ if $ab$ was sampled in epoch $t$ and zero otherwise. By construction, $\mathbb{E}(X_{ab}^t) = \mu_{ab}$ and $\mathbb{E}(Y_{ab,s}^t|X_{ab}^t = 1) = m/n$. Taking the expectation over of the random events in an epoch, we obtain the expected gradient of UMAP's optimization procedure, see Appendix \ref{app:exp_UMAP_grad} for more details:
\begin{align}
\mathbb{E}\left(g_i^t\right) &= 2 \sum_{j=1}^n \mu_{ij} \cdot \frac{\partial \mathcal{L}^a_{ij}}{\partial e_i} + \frac{d_im}{2n} \cdot \frac{\partial \mathcal{L}^r_{ij}}{\partial e_i}.\label{eq:exp_UMAP_grad_no_push_tail}
\intertext{Comparing the above closed formula for the expectation of the gradients, with which UMAP updates the low-dimensional embeddings, to the gradient of UMAP's loss function}
    \frac{\partial \mathcal{L}}{\partial e_i} &= 2\sum_{j=1}^n \mu_{ij}\cdot \frac{\partial \mathcal{L}^a_{ij}}{\partial e_i} + (1-\mu_{ij}) \cdot \frac{\partial \mathcal{L}^r_{ij}}{\partial e_i}
\end{align}
we find that the sampling procedure yields the correct weight for the attractive term in expectation, as designed. However, as noticed by~\citet{bohm2020unifying}, the negative sampling changes the weight for the repulsive term significantly. Our closed formula helps to make their qualitative arguments precise: Instead of $1-\mu_{ij}$, we have a term  $\frac{d_im}{2n}$, which depends on the hyperparameter $m$ or \texttt{negative\_sample\_rate}. Contrary to the intention of~\cite{mcinnes2018umap}, the repulsive weights are not uniform but vary with the degree of each point $d_i$, which is typically close to $\log_2(k)$ see Appendix~\ref{app:deg_distr}.
More practically, since the non-zero high-dimensional similarities are sparse, $1-\mu_{ij}$ is equal to $1$ for most $ij$. In contrast, the expected repulsive weight is typically small for large datasets as $d_i$ is of the order of $\log_2(k)$ independent of the dataset size.

Another effect of the negative sampling is that in general the expected gradient~\eqref{eq:exp_UMAP_grad_no_push_tail} does not correspond to any loss function, see Appendix~\ref{app:no_obj_function}. We remedy this by additionally pushing the embedding of a negative sample $i$ away from the embedding $e_j$, whenever $i$ was sampled as negative sample to some edge $jk$, see line~\ref{algline:push_tail} in Algorithm~\ref{alg:npUMAP_sample}. This yields the following gradient at epoch $t$
\begin{equation}\label{eq:UMAP_grad_push_tail}
    \tilde{g}^t_i= \sum_{j=1}^n \left(X_{ij}^t \cdot \frac{\partial \mathcal{L}^a_{ij}}{\partial e_i}  +X_{ji}^t \cdot \frac{\partial \mathcal{L}^a_{ji}}{\partial e_i} + X_{ij}^t \cdot \sum_{s=1}^n Y^t_{ij, s} \cdot \frac{\partial \mathcal{L}^r_{is}}{\partial e_i} + \sum_{k=1}^n X_{jk}^t Y_{jk, i}^t \cdot \frac{\partial \mathcal{L}^r_{ji}}{\partial e_i}\right),
\end{equation}
corresponding to a loss in epoch $t$ of 
\begin{equation}\label{eq:actual_loss}
    \tilde{\mathcal{L}}^t = \sum_{i,j=1}^n \left( X_{ij}^t \cdot \mathcal{L}^a_{ij} + \sum_{s=1}^n X_{ij}^tY_{ij, s}^t \cdot \mathcal{L}^r_{is}\right).
\end{equation}
Using the symmetry of $\mu_{ij}$, $\mathcal{L}^a_{ij}$ and $\mathcal{L}^r_{ij}$ in $i$ and $j$, we compute the effective loss
\begin{equation}\label{eq:eff_loss}
    \tilde{\mathcal{L}} =\mathbb{E}(\tilde{\mathcal{L}}^t) = 2 \sum_{1\leq i < j \leq n} \mu_{ij} \cdot \mathcal{L}^a_{ij} + \frac{(d_i+d_j)m}{2n} \cdot \mathcal{L}^r_{ij}.
\end{equation}
In fact, pushing also the negative samples does not affect the behavior of UMAP qualitatively, see for instance Figures~\ref{fig:toy_ring_push_tail} and~\ref{fig:c_elegans_perturbed_push_tail}.\footnote{In fact, the parametric version of UMAP~\cite{sainburg2020parametric} does include the update of negative samples.} In this light, we can treat $\tilde{\mathcal{L}}$ as the effective objective function that is optimized via SGD by UMAP's optimization procedure. It differs from UMAP's loss function~\eqref{eq:UMAP_obj_embd}, by having a drastically reduced repulsive weight of $\frac{(d_i+d_j)m}{2n}$ instead of $1- \mu_{ij}$.

\begin{figure}
    \begin{subfigure}[b]{0.43\textwidth}
        \centering
        \includegraphics[width=\linewidth]{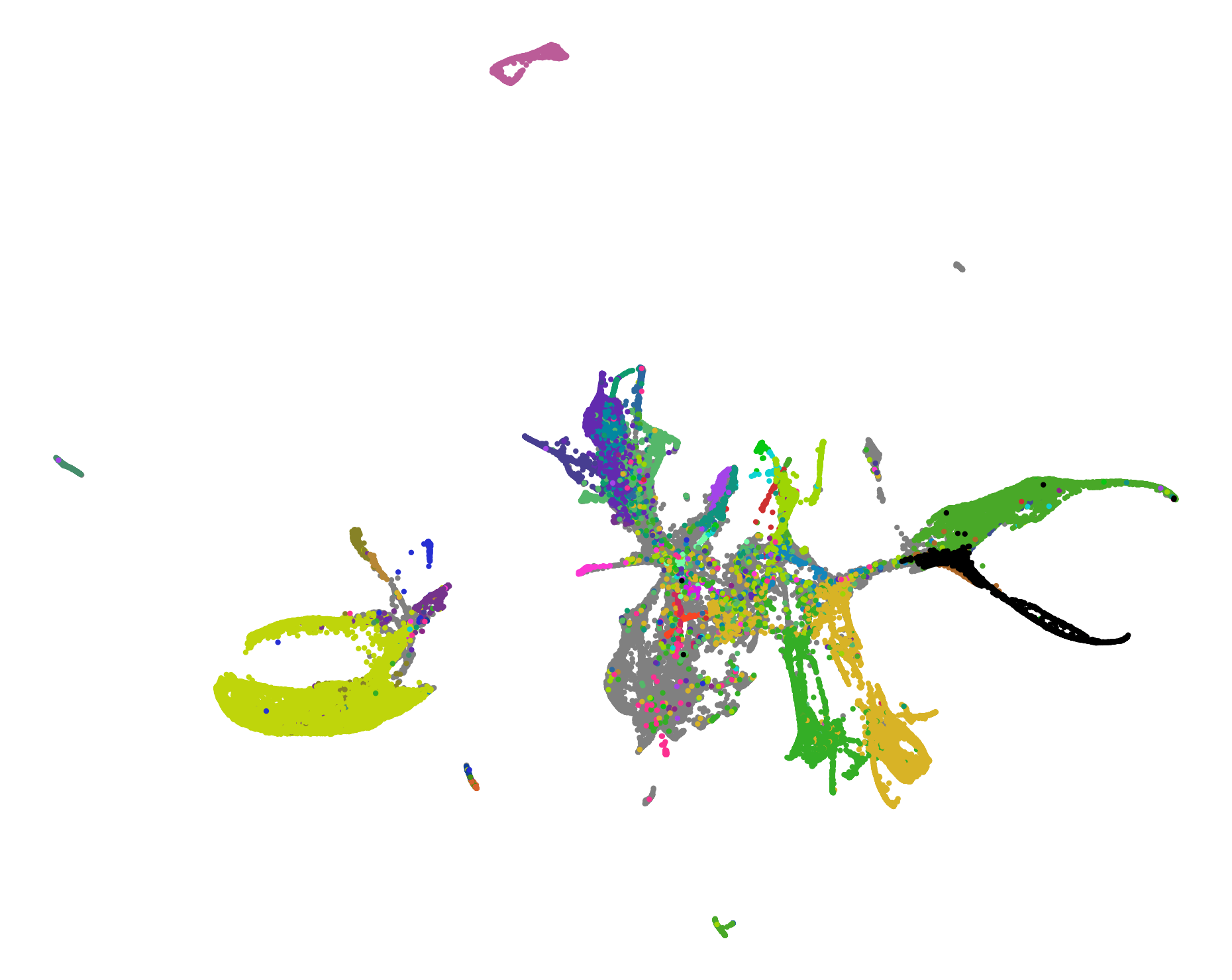}
        \caption{UMAP}
        \label{subfig:c_elegans_UMAP}
    \end{subfigure}%
        \begin{subfigure}[b]{0.32\textwidth}
        \centering
        \includegraphics[width=\linewidth]{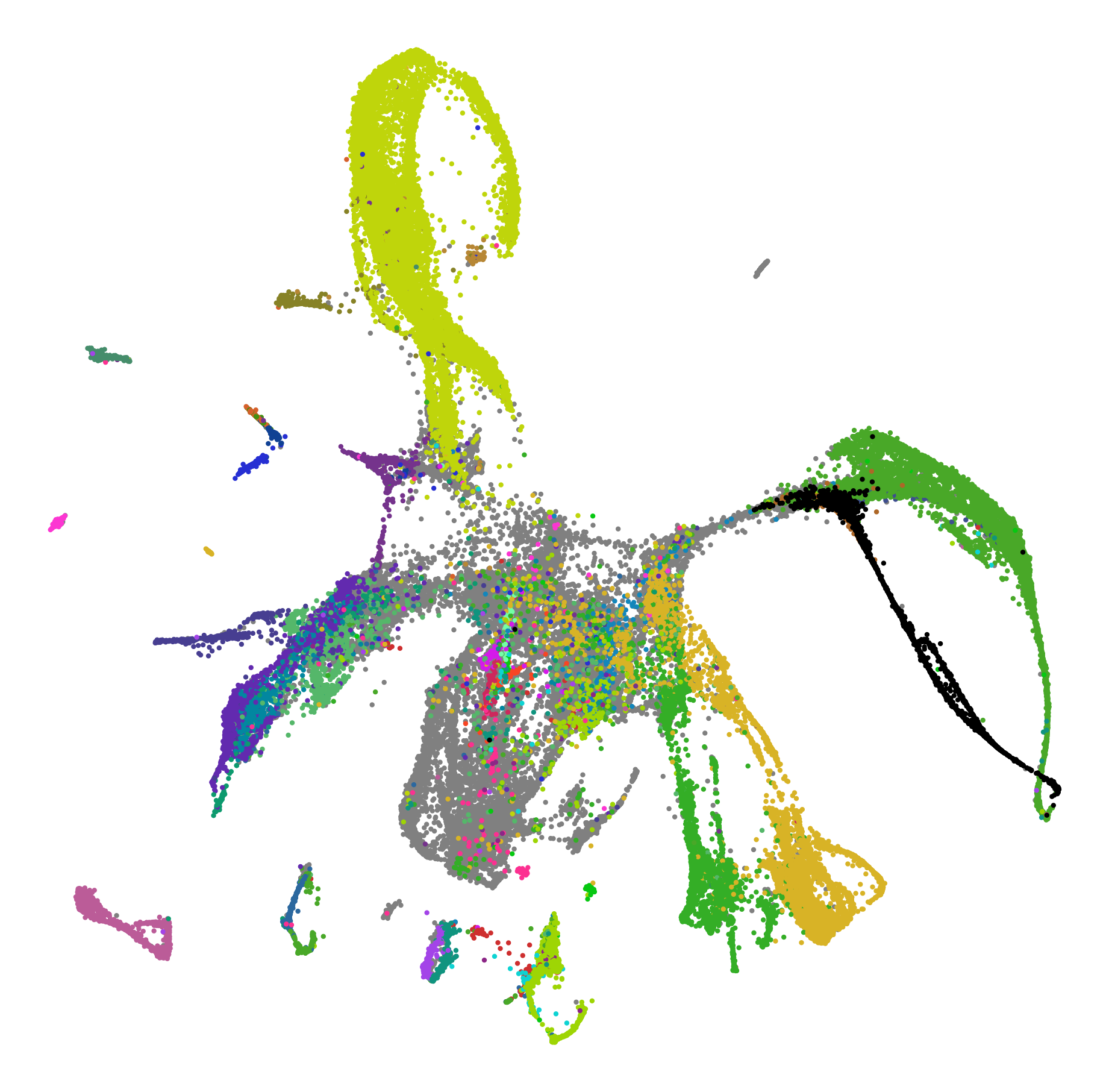}
        \caption{Inverted weights UMAP}
        \label{subfig:c_elegans_inv_UMAP}
    \end{subfigure}%
    \begin{subfigure}[b]{0.25\textwidth}
        \centering
        \includegraphics[width=\linewidth]{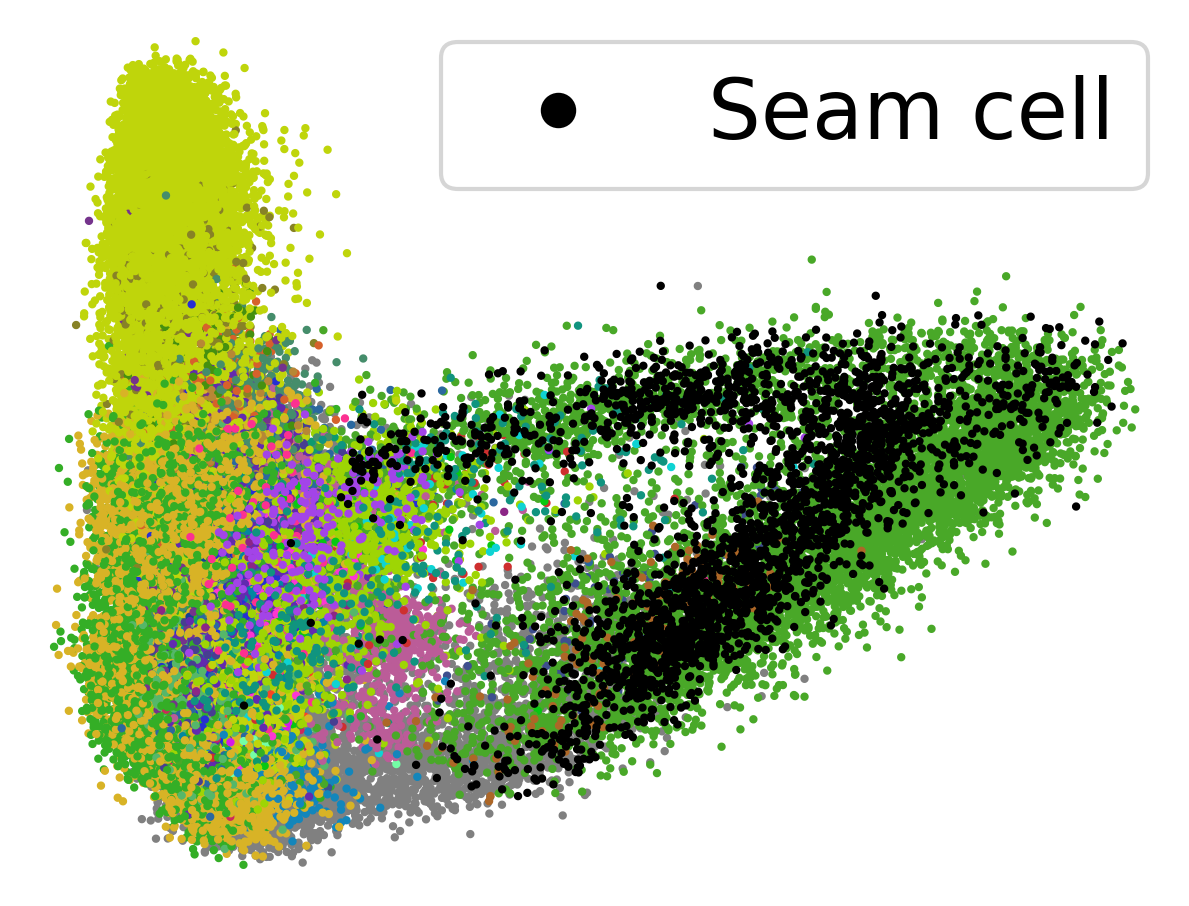}
        \caption{PCA}
        \label{subfig:c_elegans_PCA}
    \end{subfigure}

    \caption{UMAP on C. elegans data from~\cite{packer2019lineage, narayan2021assessing}. \ref{subfig:c_elegans_UMAP}~UMAP visualization with the hyperparameters of~\cite{narayan2021assessing}. Several parts of the embedding appear locally one-dimensional, for instance the seam cells. \ref{subfig:c_elegans_inv_UMAP}~Same as~\ref{subfig:c_elegans_UMAP} but with inverted positive high-dimensional similarities. The result is qualitatively similar, if not better.  \ref{subfig:c_elegans_PCA}~Two dimensional PCA of the dataset. Highlighted seam cells clearly have two dimensional variance in the PCA plot, but are over contracted to nearly a line in the UMAP plots~\ref{subfig:c_elegans_UMAP} and~\ref{subfig:c_elegans_inv_UMAP}.  Full legend with all cell types and further information can be found in Figure~\ref{fig:c_elegans_perturbed}.}
    \label{fig:c_elegans}
\end{figure}

We illustrate our analysis on the C. elegans dataset~\cite{packer2019lineage, narayan2021assessing}. We start out with a 100 dimensional PCA of the data\footnote{\label{foot:c_elegans} obtained from \url{http://cb.csail.mit.edu/cb/densvis/datasets/}. We informed the authors of our use of the dataset, which they license under CC BY-NC 2.0.} and use the cosine metric in high-dimensional space, consider local neighborhoods of 30 data points and optimize for 750 epochs as done in~\cite{narayan2021assessing}. The resulting visualization is depicted in Figure~\ref{subfig:c_elegans_UMAP}. On this dataset the average value of $1-\mu_{ij}$ is $0.9999$ but the maximal effective repulsive weight $\max_{ij}\frac{(d_i+d_j)m}{2n}$ is $0.0043$, showing the dramatic reduction of repulsion due to negative sampling. After each optimization epoch, we log our effective loss $\tilde{\mathcal{L}}$~\eqref{eq:eff_loss}, the actual loss $\tilde{\mathcal{L}^t}$~\eqref{eq:actual_loss} of each epoch computed based on the sampled (negative) pairs as well the purported UMAP loss $\mathcal{L}$~\eqref{eq:UMAP_obj_embd}.\footnote{Our code is publicly available at \url{https://github.com/hci-unihd/UMAPs-true-loss}.} We always consider the embeddings at the end of each epoch. Note that UMAP's implementation updates each embedding $e_i$ not just once at the end of the epoch but as soon as $i$ is incident to a sampled edge or sampled as a negative sample.  This difference does not change the actual loss much, see Appendix~\ref{app:implementation}. The result is plotted in Figure~\ref{fig:c_elegans_losses}. We can see that our predicted loss matches its actual counterpart nearly perfectly. While both, $\tilde{\mathcal{L}}$ and $\tilde{\mathcal{L}}^t$, agree with the attractive part of the supposed UMAP loss, its repulsive part and thus the total loss are two orders of magnitude higher. Furthermore, driven by the repulsive part, the total intended UMAP loss increases during much of the optimization process, while the actual and effective losses decrease, exemplifying that UMAP really optimizes our effective loss $\tilde{\mathcal{L}}$ ~\eqref{eq:eff_loss} instead of its purported loss $\mathcal{L}$~\eqref{eq:UMAP_obj_embd}. 

After having deduced the effective loss function of Non-Parametric UMAP, we conclude this Section with the corresponding but slightly different result for Parametric UMAP~\cite{sainburg2020parametric}:
\begin{theorem}
The effective loss function of Parametric UMAP is 
\begin{align}
    - \frac{1}{(m+1)\mu_\text{tot}}  \sum_{1\leq i < j\leq n}\kern-0.9em \mu_{ij} \log\Big(\phi\big(f_\theta(x_i), f_\theta(x_j)\big)\Big)
    +m\frac{b-1}{b} \frac{d_i d_j}{2\mu_\text{tot}} \log\Big(1-\phi\big(f_\theta(x_i), f_\theta(x_j)\big)\Big).
\end{align}
\end{theorem}
\begin{proof}
The proof can be found in Appendix~\ref{app:pUMAP} and differs from the Non-Parametric case mostly because the negative samples come from the current batch instead of the full dataset.
\end{proof}
While the exact formula differs from $\tilde{\mathcal{L}}$~\eqref{eq:eff_loss} the same analysis holds unless explicitly mentioned.

\begin{minipage}{\textwidth}
\vspace{2ex}
\begin{minipage}[t]{0.4\linewidth}
\LinesNumbered
\begin{algorithm}[H]
\SetAlgoLined
\SetAlgoNoEnd
\SetInd{0.25em}{0.25em}
\SetKwInOut{Input}{input}\SetKwInOut{Output}{output}
\DontPrintSemicolon
\Input{input similarities $\mu_{ij}$,\\ 
       initial embeddings $e_i$, \\
       number of epochs T, \\
       learning rate $\alpha$}
\Output{final embedding $e_i$}

\For{$t = 0$ \KwTo $T$}{
    \For{$ij \in \{1, \dots, n\}^2$}{
    $r \sim \text{Uniform}(0,1)$\;
        \If{$r < \mu_{ij}$}{
            $e_i = e_i - \alpha \cdot \frac{\partial \mathcal{L}^a_{ij}}{\partial e_i}$\label{algline:pull_head}\;
            $e_j = e_j - \alpha \cdot \frac{\partial \mathcal{L}^a_{ij}}{\partial e_j}$\label{algline:pull_tail}\;
            
            \For{$l=1$ \KwTo $m$}{
                $s \sim \text{Uniform}(\{1, \dots, n\})$\label{algline:neg_sample}\;
                $e_i = e_i - \alpha \cdot \frac{\partial \mathcal{L}^r_{is}}{\partial e_i}$\label{algline:push_head}\;
                \tcp*[l]{Next line is omitted in UMAP implementation, but in\-cluded for our analysis}
                /* $e_s = e_s - \alpha \cdot \frac{\partial \mathcal{L}^r_{is}}{\partial e_s}$ \label{algline:push_tail}*/ 
            }
          } 
        
    }
}
 \caption{UMAP's optimization}
 \label{alg:npUMAP_sample}
\end{algorithm}
\end{minipage}
\begin{minipage}[t]{0.59\linewidth}
   \vspace{-2ex}
        \includegraphics[width=\linewidth]{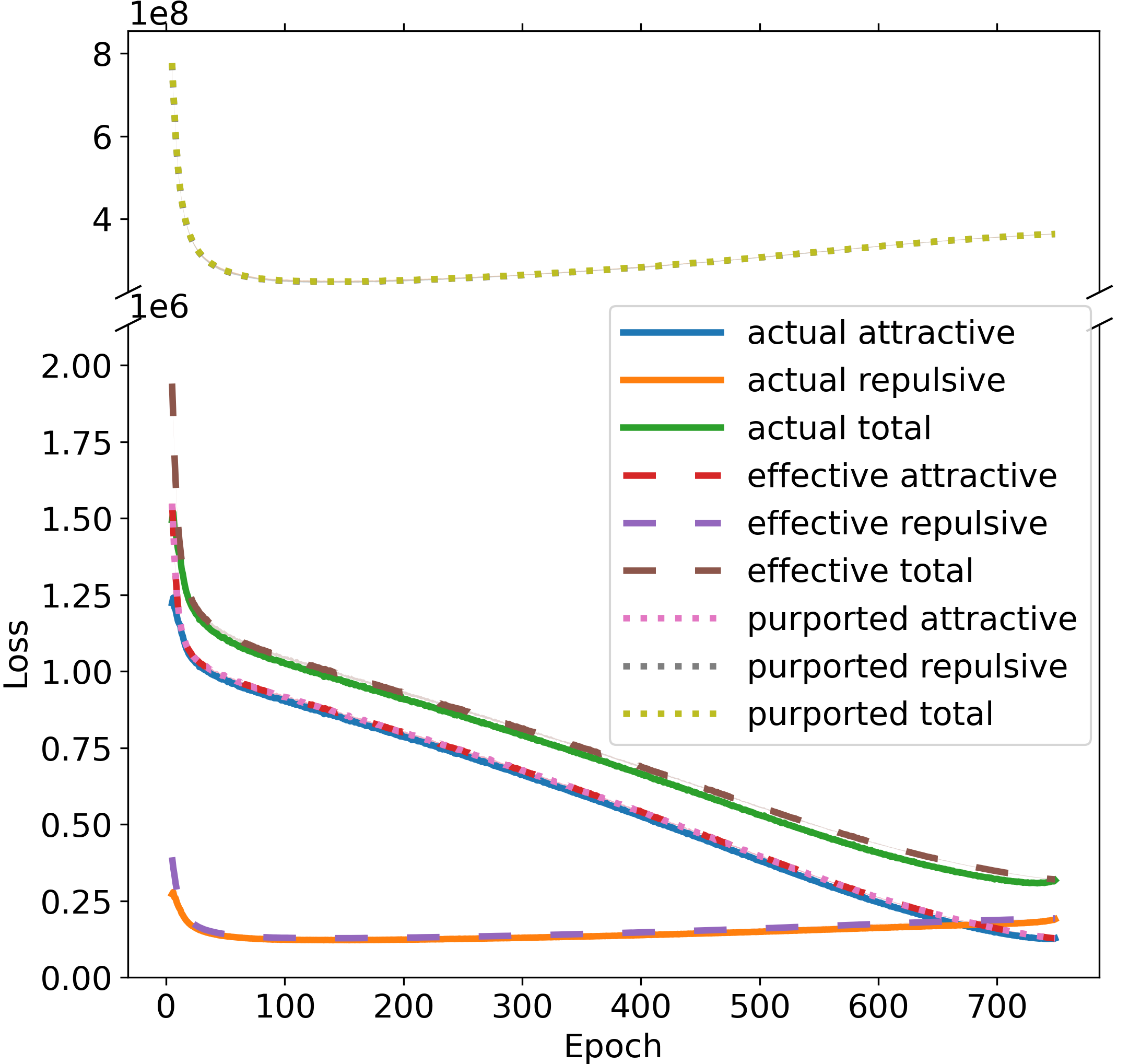}
        \captionof{figure}{Loss curves for the optimization leading to Figure~\ref{subfig:c_elegans_UMAP}. Our effective loss closely matches the actual loss on the sampled pairs, while the supposed UMAP loss~\ref{eq:UMAP_obj_embd}, which would reproduce the high-dimensional similarities, is two orders of magnitude higher. The repulsive purported loss is overlaid by the total purported loss. Average over 7 runs is plotted, barely visible shared area is one standard deviation, see also Appendix~\ref{subapp:stability}.}
        \label{fig:c_elegans_losses}
\end{minipage}

\end{minipage}

\section{True target similarities}\label{sec:target_sims}
Since the effective objective function $\tilde{\mathcal{L}}$~\eqref{eq:eff_loss} that UMAP optimizes is different from $\mathcal{L}$~\eqref{eq:UMAP_obj_embd}, we cannot hope that UMAP truly tries to find a low-dimensional embedding whose similarities reproduce the high-dimensional similarities. Nevertheless, using the effective loss $\tilde{\mathcal{L}}$, we can compute the true target similarities $\nu_{ij}^*$ which UMAP tries to achieve in embedding space. 
The effective loss $\tilde{\mathcal{L}}$ is a sum of non-normalized binary cross entropy loss functions
\begin{equation}
 -\left(\mu_{ij} \cdot \log(\nu_{ij}) + \frac{(d_i+d_j)m}{2n} \cdot \log(1- \nu_{ij})\right)
\end{equation}
which is minimal for 
\begin{equation}
    \nu_{ij}^* = \frac{\mu_{ij}}{\mu_{ij} + \frac{(d_i+d_j)m}{2n}}  
    \begin{cases}
        = 0 \text{ if } \mu_{ij} =0 \\
        \approx 1 \text{ if } \mu_{ij} >0.
    \end{cases}
\end{equation}
The approximation holds in the typical case in which $\frac{(d_i+d_j)m}{2n} \approx 0$, discussed above.
In other words, the reduced repulsion weight essentially binarizes the high-dimensional similarities. UMAP's high-dimensional similarities are non-zero exactly on the shared $k$-nearest neighbor graph edges of the high-dimensional data. Therefore, the binarization explains why~\citet{bohm2020unifying} find that using the binary weights of the shared $k$ nearest neighbor graph does not deteriorate UMAP's performance much.\footnote{\citet{bohm2020unifying} used a scaled version of the $k$NN graph, but the scaling factor cancels for the target weights.} The binarization even helps UMAP to overcome disrupted high-dimensional similarities, as long as only the edges of the shared $k$NN graph have non-zero weight. In Figure~\ref{subfig:c_elegans_inv_UMAP} we invert the original positive high-dimensional weights on the C. elegans dataset. That means that the $k$-th nearest neighbor will have higher weight than the nearest neighbor. The resulting visualization even improves on the original by keeping the layout more compact. This underpins B\"{o}hm et al.~\cite{bohm2020unifying}'s claim that the elaborate theory used to compute the high-dimensional similarities is not the reason for UMAP's practical success. In fact, we show that UMAP's optimization scheme even actively ignores most information beyond the shared $k$NN graph. In Figure~\ref{fig:c_elegans_hists}, we show histograms of the various notions of similarity for the C. elegans dataset. We see in panel~\ref{subfig:c_elegans_hist_no_inv_inv} how the binarization equalizes the positive target similarities for the original and the inverted high-dimensional similarities. 

The binary cross entropy terms in the effective loss $\tilde{\mathcal{L}}$~\eqref{eq:eff_loss} are not normalized. This leads to a different weighing of the binary cross entropy terms for each pair $ij$
\begin{align}
   \tilde{\mathcal{L}} 
   &= 2 \sum_{1\leq i < j \leq n} \mu_{ij} \cdot \mathcal{L}^a_{ij} + \frac{(d_i+d_j)m}{2n} \cdot \mathcal{L}^r_{ij}\\
   &= -2 \sum_{1\leq i < j \leq n} \left(\mu_{ij} + \frac{(d_i+d_j)m}{2n}\right) \cdot \left( \nu_{ij}^* \log(\nu_{ij}) + (1- \nu_{ij}^*) \log(1-\nu_{ij})\right). \label{eq:weighted_BCE}
 \end{align}
As $\frac{(d_i + d_j)m}{2n}$ is very small for large datasets, the term $\mu_{ij} + \frac{(d_i+d_j)m}{2n}$ is dominated by $\mu_{ij}$. Hence, the reduced repulsion not only binarizes the high-dimensional similarities, it also puts higher weight on the positive than the zero target similarities. Therefore, we can expect that the positive target similarities are better approximated by the embedding similarities, than the zero ones. Indeed, panel~\ref{subfig:c_elegans_hist_pos} shows that the low-dimensional similarities match the positive target similarities very well, as expected from the weighted BCE reading of the effective loss function~\eqref{eq:weighted_BCE}. 

\subsection{Explaining artifacts in UMAP visualizations}\label{subsec:interpret_toy_ring}
We conclude this section by explaining the observed artifacts of UMAP's visualization in Figures~\ref{fig:toy_ring} and~\ref{fig:c_elegans} in the light of the above analysis. The normal UMAP optimization contracts the ring in Figure~\ref{subfig:toy_ring_UMAP} even when initialized at the original layout (Figure~\ref{subfig:toy_ring_original}) because the reduced repulsion yields nearly binary target similarities. All pairs that are part of the $k$NN graph not only want to be sufficiently close that their high-dimensional similarity is reproduced, but so close that their similarity is one. The fact that the effective loss weighs the terms with target similarity near one much more than those with target similarity near zero reinforces this trend. As a result, the ring gets contracted to a circle. The same argument applies to the over contracted parts of the UMAP visualization of the C. elegans dataset in Figure~\ref{fig:c_elegans}. Our framework can also explain the opposite behavior of UMAP when the dense similarities are used as input similarities, see Figure~\ref{subfig:toy_ring_dense}. In this setting, the average degree of a node is about $100$. With a \texttt{negative\_sample\_rate} of $5$ and a dataset size of $n=1000$ this yields repulsive weights of about $\frac{(d_i+d_j)m}{2n} \approx 0.5$. Thus, we increase the repulsion on pairs with high input similarity, but decrease it on pairs with low input similarity. The target similarities are lower (larger) than the input similarities if the latter are larger (lower) than $0.5$. Consequently, we can expect embedding points to increase their distance to nearest neighbors, but distant points to move closer towards each other. This is what we observe in Figure~\ref{subfig:toy_ring_dense}, where the width of the ring has increased and the ring curves to bring distant points closer together.

\begin{figure}
    \begin{subfigure}{0.33\textwidth}
        \centering
        \includegraphics[width=0.9\linewidth]{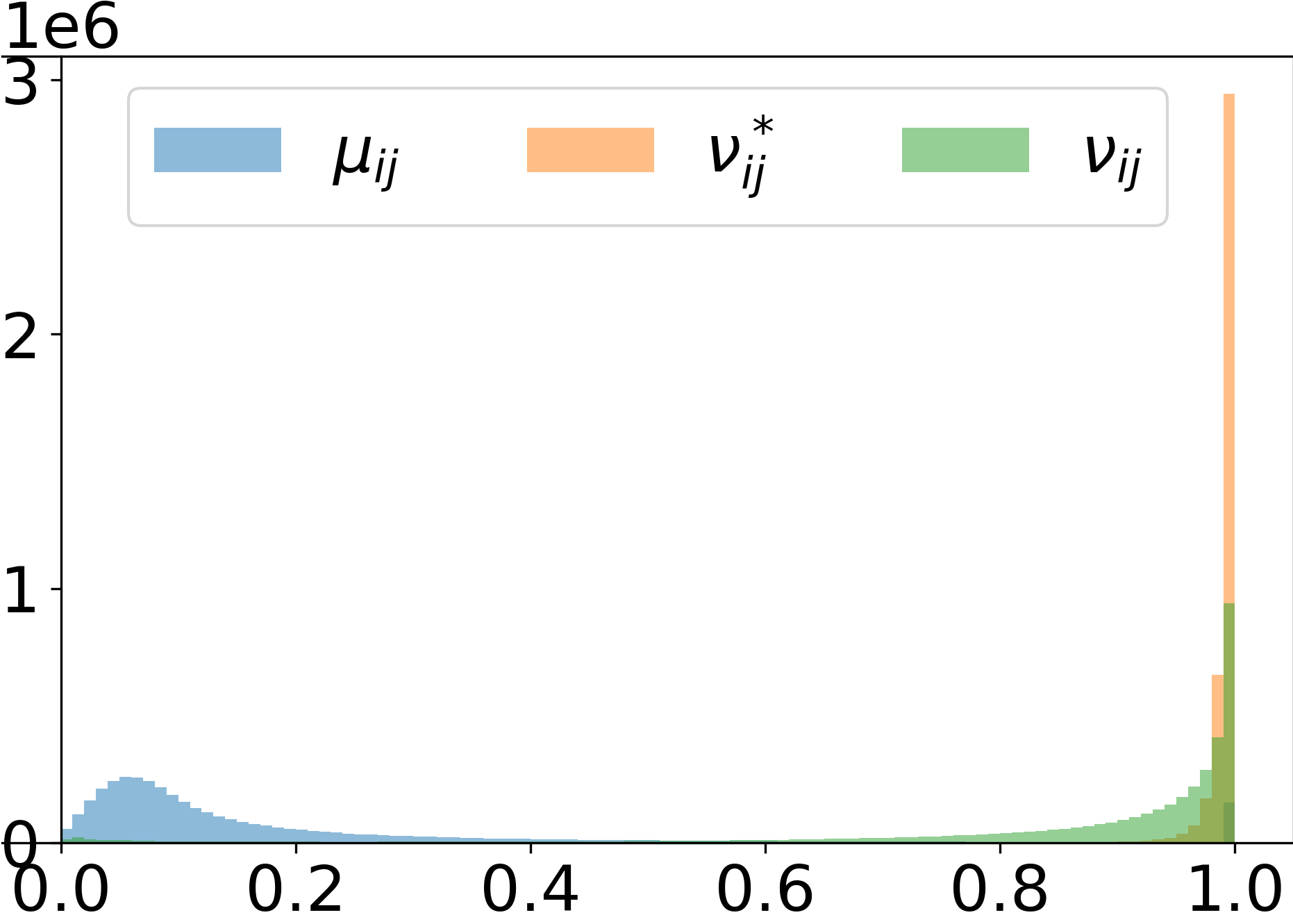}
        \caption{Similarities for $\mu_{ij}>0$}
        \label{subfig:c_elegans_hist_pos}
    \end{subfigure}%
    \begin{subfigure}{0.33\textwidth}
        \centering
        \includegraphics[width=0.9\linewidth]{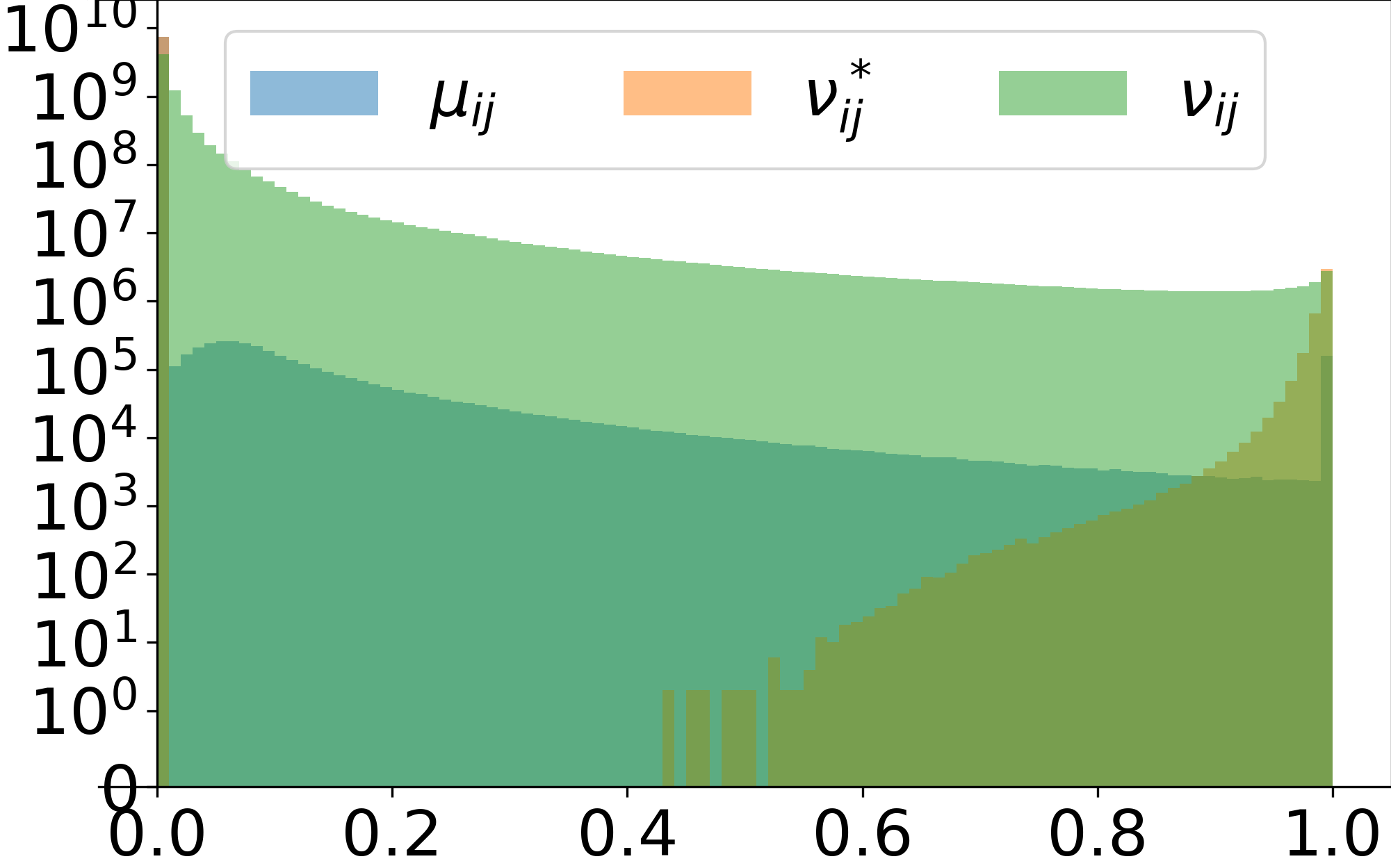}
        \caption{All similarities}
        \label{subfig:c_elegans_hist_all}
    \end{subfigure}%
    \begin{subfigure}{0.33\textwidth}
        \centering
        \includegraphics[width=0.9\linewidth]{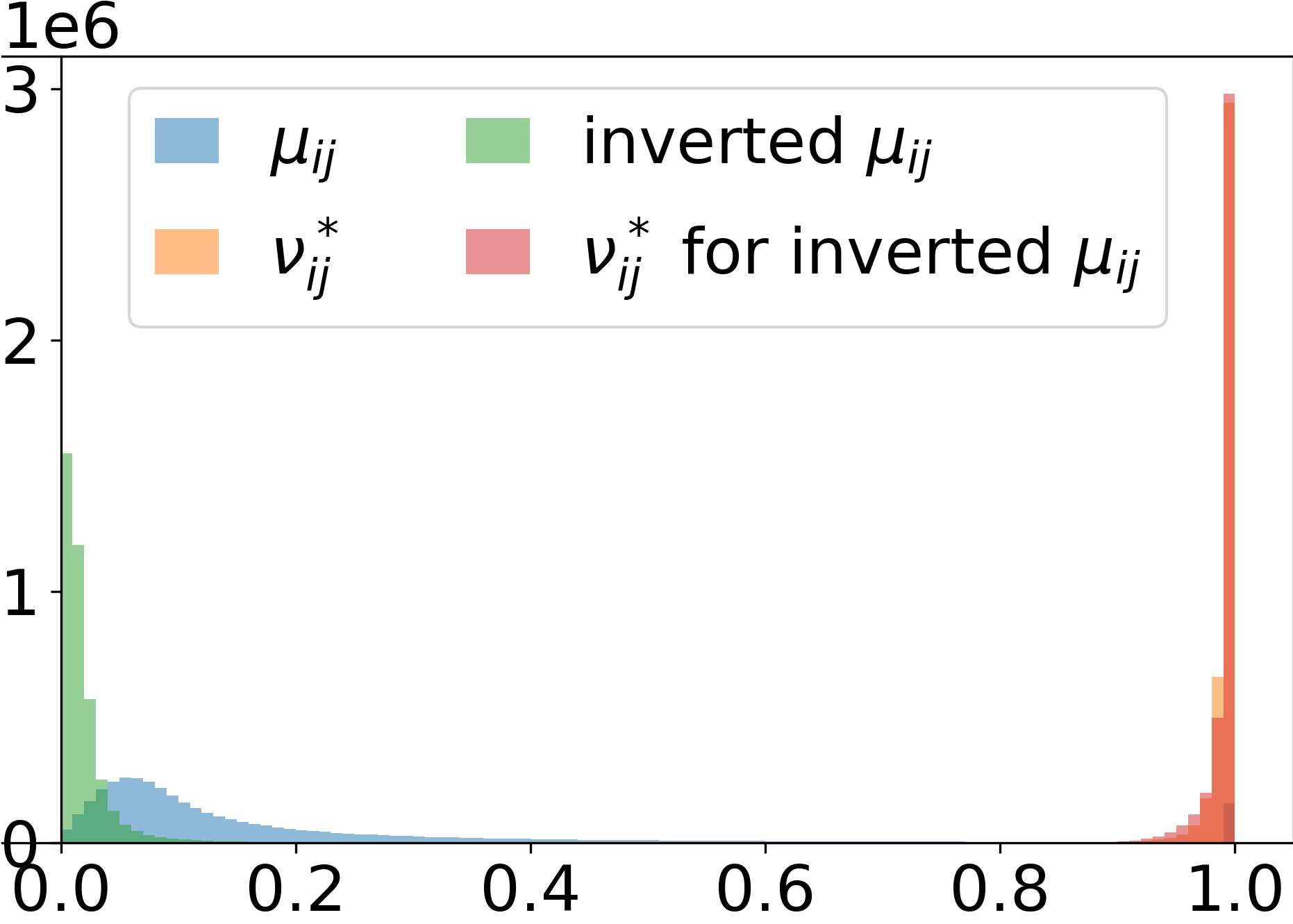}
        \caption{Original and inverted similarities for $\mu_{ij} >0$}
        \label{subfig:c_elegans_hist_no_inv_inv}
    \end{subfigure}%
    
    \caption{Histograms of high-dimensional ($\mu_{ij}$), target ($\nu^*_{ij}$) and low-dimensional ($\nu_{ij}$) similarities on the C. elegans dataset~\cite{packer2019lineage, narayan2021assessing}. The similarities of UMAP's low-dimensional embedding reproduce the target similarities instead of the high-dimensional ones. \ref{subfig:c_elegans_hist_pos}~Only similarities for pairs with positive high-dimensional similarity are shown. Compared to the high-dimensional similarities, the target similarities are heavily skewed towards one and closely resemble the low-dimensional ones. \ref{subfig:c_elegans_hist_all}~All similarities and depicted on a logarithmic scale. There are many more pairs that have zero high-dimensional similarity than positive high-dimensional similarity. \ref{subfig:c_elegans_hist_no_inv_inv}~Comparison of similarities for pairs of positive high-dimensional similarities for the original UMAP and the inverted similarities. While the histograms of the high-dimensional similarities differ noticeably, their target similarities do not. The binarization essentially ignores all information beyond the shared $k$NN graph.}
    \label{fig:c_elegans_hists}
\end{figure}

\section{Discussion}\label{sec:discussion}
By deriving UMAP's true loss function and target similarities, we are able to explain several peculiar properties of UMAP visualizations. According to our analysis, UMAP does not aim to reproduce the high-dimensional UMAP similarities in low dimension but rather the binary shared $k$NN graph of the input data. This raises the question just what part of UMAP's optimization leads to its excellent visualization results. Apparently, the exact formula for the repulsive weights is not crucial as it differs for Non-Parametric UMAP and Parametric UMAP while both produce similarly high quality embeddings. A first tentative step towards an explanation might be the different weighing of the BCE terms in the effective loss function~\eqref{eq:weighted_BCE}. Focusing more on the similar rather than the dissimilar pairs might help to overcome the imbalance between an essentially linear number of attractive and a quadratic number of repulsive pairs. Inflated attraction was found beneficial for $t$-SNE as well, in the form of early exaggeration~\cite{linderman2019clustering}. 

Put another way, the decreased repulsive weights result in comparable total attractive and repulsive weights, which might facilitate the SGD based optimization. Indeed, the total attractive weight in UMAP's effective loss functions is $2\mu_\text{tot} = \sum_{i,j = 1}^n \mu_{ij}$ and the total repulsive weight roughly amounts to $m\mu_\text{tot} = \sum_{i,j = 1} \frac{(d_i+d_j)m}{2n}$ for Non-Parametric UMAP and to $2m\mu_\text{tot}\frac{b-1}{b}$ for Parametric UMAP. For the default value of $m=5$, the total attractive and repulsive weights are of roughly the same order of magnitude. Moreover, we observe in Figure~\ref{fig:c_elegans_losses} that the resulting attractive and repulsive losses are also of comparable size. Using UMAP's purported loss function, however, would yield dominating repulsion. A more in-depth investigation as to why exactly balanced attraction and repulsion is beneficial for a useful embedding is interesting and left for future work.

\section{Conclusion}\label{sec:conclusion}
In this work, we investigated UMAP's optimization procedure in depth. In particular, we computed UMAP's effective loss function analytically and found that it differs slightly between the non-parametric and parametric versions of UMAP and significantly from UMAP's alleged loss function. The optimal solution of the effective loss function is typically a binarized version of the high-dimensional similarities. This shows why the sophisticated form of the high-dimensional UMAP similarities does not add much benefit over the shared $k$NN graph. Instead, we conjecture that the resulting balance between attraction and repulsion is the main reason for UMAP's great visualization capability. Our analysis can explain some artifacts of UMAP visualizations.

\section*{Acknowledgements}\label{sec:acknowledgements}
Supported, in part, by Informatics for Life funded by the Klaus Tschira Foundation.

\bibliographystyle{abbrvnat}
\bibliography{references}

\clearpage
\appendix
\section{Parametric UMAP's sampling and effective loss function}\label{app:pUMAP}
In Parametric UMAP~\cite{sainburg2020parametric} the embeddings are not directly optimized. Instead a parametric function, a neural network, is trained to map the input points to embedding space. As usual, a mini-batch of data points is fed through the neural network at each training iteration; the loss is computed for this mini-batch and then the parameters of the neural network are updated via stochastic gradient descent. To avoid the quadratic complexity of the repulsive term a sampling strategy is employed, sketched in Algorithm~\ref{alg:PUMAP_sample}. There are three differences to the optimization scheme of Non-Parametric UMAP: First, since automatic differentiation is used, not only the head of a negative sample edge is repelled from the tail but both repel each other. Second, the same number of edges are sampled in each epoch. Third, since only the embeddings of the current mini-batch are available, negative samples are  produced not from the full dataset but only from within the non-uniformly assembled batch. This leads to a different repulsive weight for Parametric UMAP as described in 
\begin{theorem}
The expected loss function of Parametric UMAP is 
\begin{equation}
    - \frac{1}{2(m+1)\mu_\text{tot}}  \sum_{i,j=1}^n \mu_{ij}\cdot \log\Big(\phi\big(f_\theta(x_i), f_\theta(x_j)\big)\Big) +  m\frac{b-1}{b} \frac{d_i d_j}{2\mu_\text{tot}}\cdot \log\Big(1-\phi\big(f_\theta(x_i), f_\theta(x_j)\big)\Big),
\end{equation}
where $b$ is the batch size, $m$ the \texttt{negative\_sample\_rate} and $f_\theta$ the parametric embedding function.
\end{theorem}
\begin{proof}
Let $P_{ij}$ be the random variable for the number of times that edge $ij$ is sampled into the batch $B$ of some iteration $t$. Let further $N_{ij}$ be the random variable holding the number of negative sample pairs $ij$ in that epoch. Then the loss at iteration $t$ is given by 
\begin{equation}
    \mathcal{L}^t = - \frac{1}{(m+1)b} \sum_{i,j=1}^n P_{ij}\cdot \log(\phi(f_\theta(x_i), f_\theta(x_j))) + N_{ij} \cdot \log(1-\phi(f_\theta(x_i), f_\theta(x_j)))
\end{equation}
To compute the expectation of this loss, we need to find the expectations of the $P_{ij}$'s and $N_{ij}$'s. The edges in batch $B$ are sampled independently with replacement from the categorical distribution over all edges with probability proportional to the high-dimensional similarities. Thus, $P_{ij}$ follows the multinomial distribution $\text{Mult}(b,\{\frac{\mu_{ab}}{2\mu_\text{tot}}\}_{a,b=1, \dots, n}\})$ and $\mathbb{E}(P_{ij}) = \frac{b \mu_{ij}}{2\mu_\text{tot}}$.

To get the negative sample pairs, each entry of the heads $B_h$ and tails $B_t$ in $B$ is repeated $m$ times. We introduce the random variables $H_a$ and $T_a$ for $a=1,\dots, n$, representing the number of occurrences of $a$ among the repeated heads and tails. $N_{ij}$ counts how often the sampled permutation of the repeated tails assigns a tail $j$ to a head $i$. This can be viewed as selecting a tail from $mB_t$ (tails repeated $m$ times) for each of the $H_i$ heads $i$ without replacement. There are $T_j$ tails that lead to a negative sample pair $ij$. Therefore, $N_{ij}$ follows a hypergeometric distribution $\text{Hyp}(mb, H_i, T_j)$. So, $\mathbb{E}_\pi(N_{ij})= \frac{H_iT_j}{mb}$. We have 
\begin{equation}
H_i = m\cdot \sum_b P_{ib} \text{ and } T_j = m\cdot \sum_a P_{aj}.
\end{equation}
Since the multinomially distributed $P_{ab}$'s have covariance $\text{Cov}(P_{ab}, P_{a', b'}) = -b\frac{\mu_{ab}\mu_{a'b'}}{4\mu_\text{tot}^2}$, we get 
\begin{equation}
    \mathbb{E}_B(P_{ab}P_{a'b'})= \text{Cov}(P_{ab}, P_{a'b'}) + \mathbb{E}_B(P_{ab})\mathbb{E}_B(P_{a'b'}) = b(b-1)\frac{\mu_{ab}\mu_{a'b'}}{4\mu_\text{tot}^2}.
\end{equation}
With this we compute the expectation of $\mathbb{E}_\pi(N_{ij})$ with respect to the batch assembly as
\begin{align}
    \mathbb{E}_B(\mathbb{E}_\pi(N_{ij})) &= \frac{1}{mb}\mathbb{E}_B(H_iT_j) \notag \\
                                      &= \frac{1}{mb}\mathbb{E}_B\left( m\sum_{b=1}^n P_{ib} \cdot m\sum_{a=1}^n P_{aj}\right) \notag \\
                                      &= \frac{m}{b} \sum_{a,b =1}^n \mathbb{E}_B(P_{ib}P_{aj}) \notag \\
                                      &= \frac{m}{b} \sum_{a,b =1}^n b(b-1)\frac{\mu_{ib}\mu_{aj}}{4\mu_\text{tot}^2} \notag\\
                                      &= m(b-1) \frac{d_i d_j}{4\mu_\text{tot}^2}.
\end{align}
Finally, as the random process of the batch assembly is independent of the choice of the permutation, we can split the total expectation up and get the expected loss 
\begin{align}
    &\mathbb{E}_{(B,\pi)}(\mathcal{L}^t) \notag \\
    &\quad= \mathbb{E}_B\mathbb{E}_\pi\left(- \frac{1}{(m+1)b} \sum_{i,j=1}^n P_{ij}\cdot \log(\phi(f_\theta(x_i), f_\theta(x_j))) + N_{ij} \cdot \log(1-\phi(f_\theta(x_i), f_\theta(x_j)))\right) \notag \\
    &\quad=- \frac{1}{(m+1)b}  \sum_{i,j=1}^n \mathbb{E}_B(\mathbb{E}_\pi(P_{ij}))\cdot \log(\phi(f_\theta(x_i), f_\theta(x_j))) \notag \\
    & \hspace{2.9cm}+ \mathbb{E}_B\mathbb{E}_\pi(N_{ij}) \cdot \log(1-\phi(f_\theta(x_i), f_\theta(x_j)))\notag \\
    &\quad = - \frac{1}{(m+1)b}  \sum_{i,j=1}^n \frac{b\mu_{ij}}{2\mu_\text{tot}}\cdot \log(\phi(f_\theta(x_i), f_\theta(x_j))) \notag \\
    & \hspace{2.9cm}+  m(b-1)\frac{d_i d_j}{4\mu_\text{tot}^2}\cdot \log(1-\phi(f_\theta(x_i), f_\theta(x_j))) \notag \\
    &\quad = - \frac{1}{2(m+1)\mu_\text{tot}}  \sum_{i,j=1}^n \mu_{ij}\cdot \log\Big(\phi\big(f_\theta(x_i), f_\theta(x_j)\big)\Big) \notag \\
    & \hspace{2.9cm}+  m\frac{b-1}{b} \frac{d_i d_j}{2\mu_\text{tot}}\cdot \log\Big(1-\phi\big(f_\theta(x_i), f_\theta(x_j)\big)\Big).
\end{align}
\end{proof}

\LinesNumbered
\begin{algorithm}[tb]
\SetAlgoLined
\SetAlgoNoEnd
\SetKwInOut{Input}{input}\SetKwInOut{Output}{output}
\DontPrintSemicolon
\Input{high-dimensional similarities $\mu_{ij}$, number of epochs T, learning rate $\alpha$, embedding network $f_\theta$, batch size $b$}
\Output{final embeddings $e_i$}
\For{$\tau = 0$ \KwTo $T$}{
    \emph{Assemble batch}\;
    $B_h, B_t = [\,],[\,] $\tcp*{Initialize empty mini-batches for heads and tails }
    \For(\tcp*[f]{Sample edge by input similarity and add to batch}){$\beta=1$ \KwTo $b$}{
        $ij \sim \text{Cat}(\{1, \dots, n\}^2, \{\frac{\mu_{ab}}{2\mu_\text{tot}}\}_{a,b=1, \dots, n}\})$\;
        $B_h\text{.append}(f_\theta(x_i))$\;
        $B_t\text{.append}(f_\theta(x_j))$\;
    }
    
    \emph{Compute loss}\;
    $l = 0$\;
    \For(\tcp*[f]{Add attractive loss for sampled edges} ){$\beta=1$ \KwTo $b$}{
        $l = l + \mathcal{L}^a(B_h[\beta], B_t[\beta])$\;
    }
    $\pi \sim \text{Uniform}(\text{permutations of }\{1, \dots, m\cdot b\})$\;
    \For(\tcp*[f]{Add repulsive loss between negative samples}){$\beta=1$ \KwTo $mb$}{
        $l = l + \mathcal{L}^r(mB_h[\beta], mB_t[\pi(\beta)])$\tcp*{$mB$ repeats  $B$ $m$ times}
    }
    $l = \frac{l}{(m+1)b}$\;
    \emph{Update parameters}\;
    $\theta =  \theta - \alpha \cdot \nabla_\theta l$\;
        
}
\KwRet{$f_\theta(x_1), \dots, f_\theta(x_n)$}
\caption{Parametic UMAP's sampling based optimization}
 \label{alg:PUMAP_sample}
\end{algorithm}

\section{UMAP degree distributions}\label{app:deg_distr}
Before symmetrization, the degree of each node $\vec{d}_i = \sum_{j=1}^n \mu_{i \to j}$ equals $\log_2(k)$  due to UMAP's uniformity assumption. For UMAP's default value of $k=15$ this is $\approx 3.9$, for $k=30$ as for the C.elegans dataset $\approx 4.9$. Symmetrizing changes the degree in a dataset-dependent way. Since  $\max(a, b) \leq a +b - ab$ for $a,b \in [0, 1]$, the symmetric degrees $d_i = \sum_{j=1}^n \mu_{ij}$ are lower bounded by $\log_2(k)$. Empirically, we find that the degree distribution is fairly peaked close to this lower bound, see Figure~\ref{fig:degree_hist}.

In the shared $k$NN graph each node has degree at least $k$. Empirically, the degree distribution is fairly peaked at this lower bound, see Figure~\ref{fig:degree_hist_kNN}.

\begin{figure}
    \centering
    \includegraphics[width=0.5\linewidth]{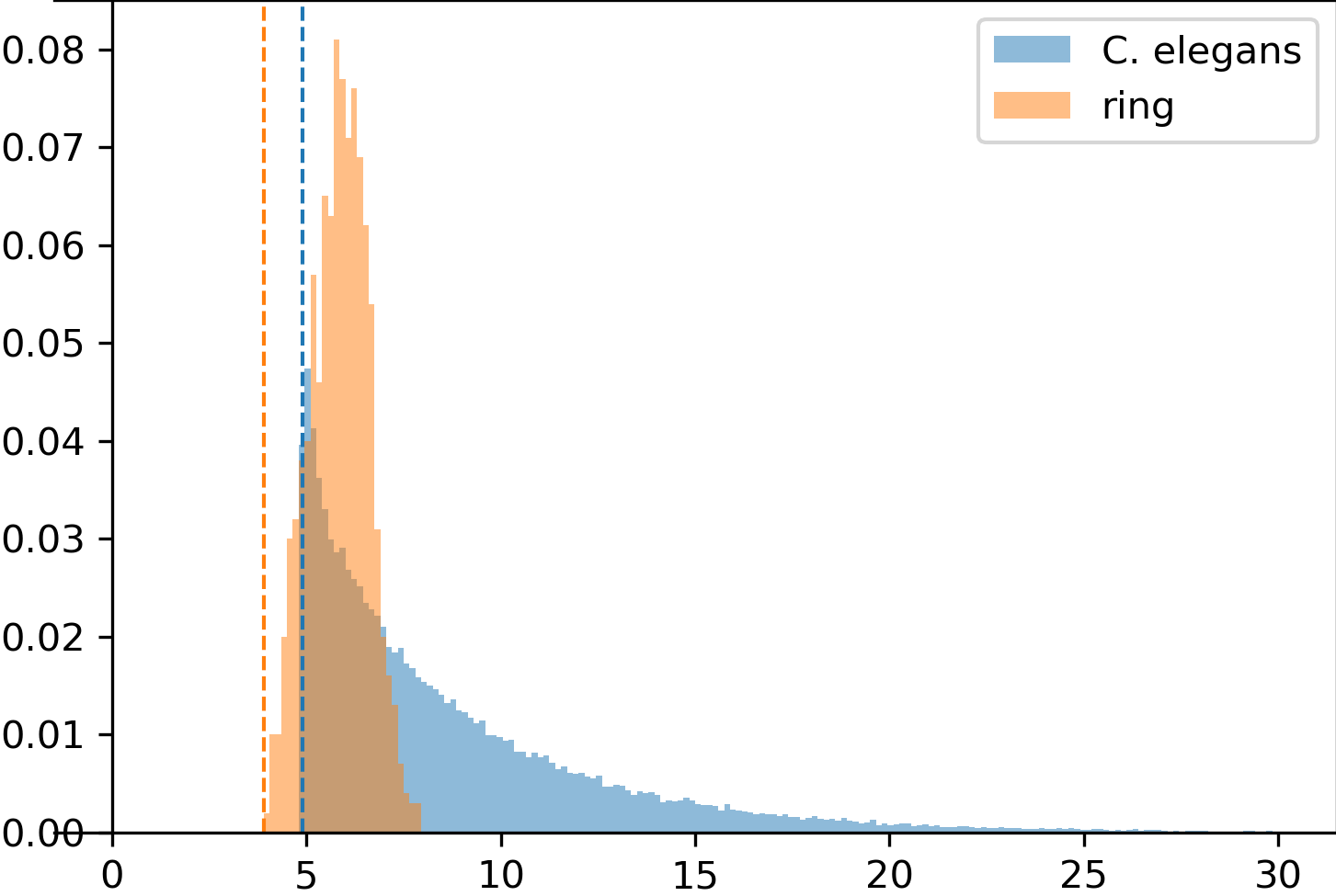}
    \caption{Histogram over the UMAP degree distributions for the toy ring and the C. elegans datasets. Both distributions are fairly peaked close to their lower bound $\log_2(k)$, highlighted as dashed line.}
    \label{fig:degree_hist}
\end{figure}

\begin{figure}
    \centering
    \includegraphics[width=0.5\linewidth]{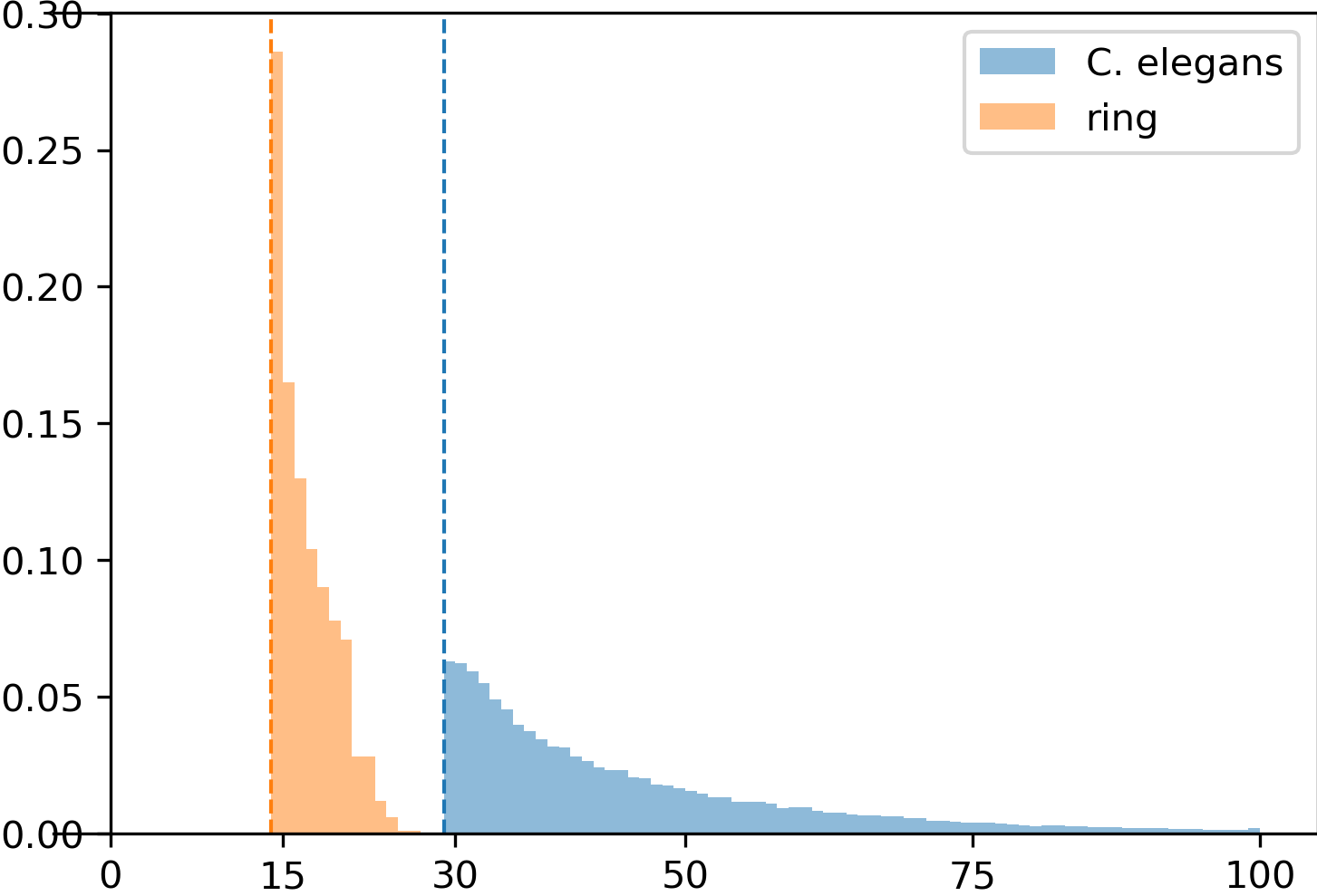}
    \caption{Histogram over the degree distribution in the shared $k$NN graph for the toy ring and the C. elegans datasets. Both distributions are fairly peaked close to their lower bound $k-1$, highlighted as dashed line. Since UMAP's implementation considers a points its first nearest neighbor, but the $\mu_{ii}$ are set to zero, the degree is one lower than the intended number of nearest neighbors $k$.}
    \label{fig:degree_hist_kNN}
\end{figure}

\section{Computing the expected gradient of UMAP's optimization procedure}\label{app:exp_UMAP_grad}
In this appendix, we elaborate the computation that leads to equation \eqref{eq:exp_UMAP_grad_no_push_tail}:
\begin{align}
    \mathbb{E}\left(g_i^t\right) &= \mathbb{E}\left( \sum_{j=1}^n X_{ij}^t \cdot \frac{\partial \mathcal{L}^a_{ij}}{\partial e_i}   +X_{ji}^t \cdot \frac{\partial \mathcal{L}^a_{ji}}{\partial e_i}  + X_{ij}^t \cdot \sum_{s=1}^n Y^t_{ij, s} \cdot \frac{\partial \mathcal{L}^r_{is}}{\partial e_i}\right) \notag\\
    &=  \sum_{j=1}^n \left(\mathbb{E}(X_{ij}^t) \cdot \frac{\partial \mathcal{L}^a_{ij}}{\partial e_i} +\mathbb{E}(X_{ji}^t) \cdot \frac{\partial \mathcal{L}^a_{ji}}{\partial e_i}  + \sum_{s=1}^n \mathbb{E}(X_{ij}^tY^t_{ij, s}) \cdot \frac{\partial \mathcal{L}^r_{is}}{\partial e_i} \right)\notag \\
    &=    \sum_{j=1}^n \mu_{ij} \cdot \frac{\partial \mathcal{L}^a_{ij}}{\partial e_i} +  \mu_{ji} \cdot \frac{\partial \mathcal{L}^a_{ji}}{\partial e_i} + \sum_{s=1}^n \sum_{j=1}^n\frac{\mu_{ij}m}{n} \cdot \frac{\partial \mathcal{L}^r_{is}}{\partial e_i} \notag\\
    &= 2 \sum_{j=1}^n  \mu_{ij} \cdot \frac{\partial \mathcal{L}^a_{ij}}{\partial e_i} + \frac{d_i m}{2n} \cdot \frac{\partial \mathcal{L}^r_{ij}}{\partial e_i}.
\end{align}
From line 2 to 3, we computed $\mathbb{E}(X_{ij}^tY^t_{ij, s}) = \mathbb{E}_{X_{ij}^t}\big(X_{ij}^t\cdot \mathbb{E}(Y_{ij,s}^t|X_{ij}^t)\big) = \frac{\mu_{ij}  m}{n}$ and from line 3 to 4 we used the symmetry of $\mu_{ij}$ and $\mathcal{L}^a_{ij}$ and collected the high-dimensional similarities $\sum_j \mu_{ij}$ into the degree $d_i$.

\section{UMAP's update rule has no objective function}\label{app:no_obj_function}
In this appendix, we show that the expected gradient update in UMAP's optimization scheme does not correspond to any objective function. Recall that the expected update of an embedding $e_i$ in UMAP's optimization scheme~\eqref{eq:exp_UMAP_grad_no_push_tail} is
\begin{equation}
    \mathbb{E}\left(g_i^t\right) = 2 \sum_{j=1}^n \mu_{ij} \cdot \frac{\partial \mathcal{L}^a_{ij}}{\partial e_i} + \frac{d_im}{2n} \cdot \frac{\partial \mathcal{L}^r_{ij}}{\partial e_i}
\end{equation}
It is continuously differentiable unless two embedding points coincide. Therefore, if it had an antiderivative, that would be twice continuously differentiable at configurations where all embeddings are pairwise distinct and thus needs to have a symmetric Hessian at these points. However, we have
\begin{align}
\label{eq:second_order_partial}
    \frac{\partial \mathbb{E}\left(\frac{\partial \mathcal{L}^t}{\partial e_i}\right)}{\partial e_j } &= 2 \mu_{ij} \cdot \frac{\partial^2 \mathcal{L}^a_{ij}}{\partial e_j \partial e_i} + \frac{d_im}{2n} \cdot \frac{\partial \mathcal{L}^r_{ij}}{\partial e_j \partial e_i} \notag \\
    \frac{\partial \mathbb{E}\left(\frac{\partial \mathcal{L}^t}{\partial e_j}\right)}{\partial e_i } &= 2 \mu_{ij} \cdot \frac{\partial^2 \mathcal{L}^a_{ij}}{\partial e_i \partial e_j} + \frac{d_jm}{2n} \cdot \frac{\partial \mathcal{L}^r_{ij}}{\partial e_i \partial e_j}.
\end{align}
Since $\mathcal{L}^a_{ij}$ and $\mathcal{L}^r_{ij}$ are themselves twice continuously differentiable, their second order partial derivatives are symmetric. But this makes the two expressions in equation~\eqref{eq:second_order_partial} unequal unless $d_i$ equals $d_j$. 

The problem is that negative samples themselves are not updated, see commented line~\ref{algline:push_tail} in Algorithm~\ref{alg:npUMAP_sample}. We suggest to remedy this by pushing the embedding of a negative sample $i$ away from the embedding node $e_j$, whenever $i$ was sampled as negative sample to some edge $jk$ This yields the gradient in equation~\eqref{eq:UMAP_grad_push_tail} at epoch $t$.

\section{Implementation Details}\label{app:implementation}
To deal with the quadratic complexity when computing all dense low-dimensional similarities $\nu_{ij}$, we used the Python package PyKeOps~\cite{charlier2020kernel} that parallelizes the computations on the GPU.

To guard us against numerical instabilities from $\log$, we always use \mbox{$\log(\min(x+0.0001, 1))$} instead of $\log(x)$.

When computing the various loss terms for UMAP, we always use the embeddings after each full epoch. The embeddings in UMAP are updated as soon as the an incident edge is sampled. Thus, an embedding might be updated several times during an epoch and gradient computations always use the current embedding, which might differ slightly from the embedding after the full epoch. Logging the loss given the embeddings at the time of each individual update yields as slightly lower attractive loss term, see Figure~\ref{fig:c_elegans_losses_during_epoch}.

Our description of UMAP's implementation is based on the original paper~\cite{mcinnes2018umap} and version 0.5.0 of the umap-learn package.\footnote{https://github.com/lmcinnes/umap}

 \begin{figure}[tb]
    \centering
    \includegraphics[width=0.5\linewidth]{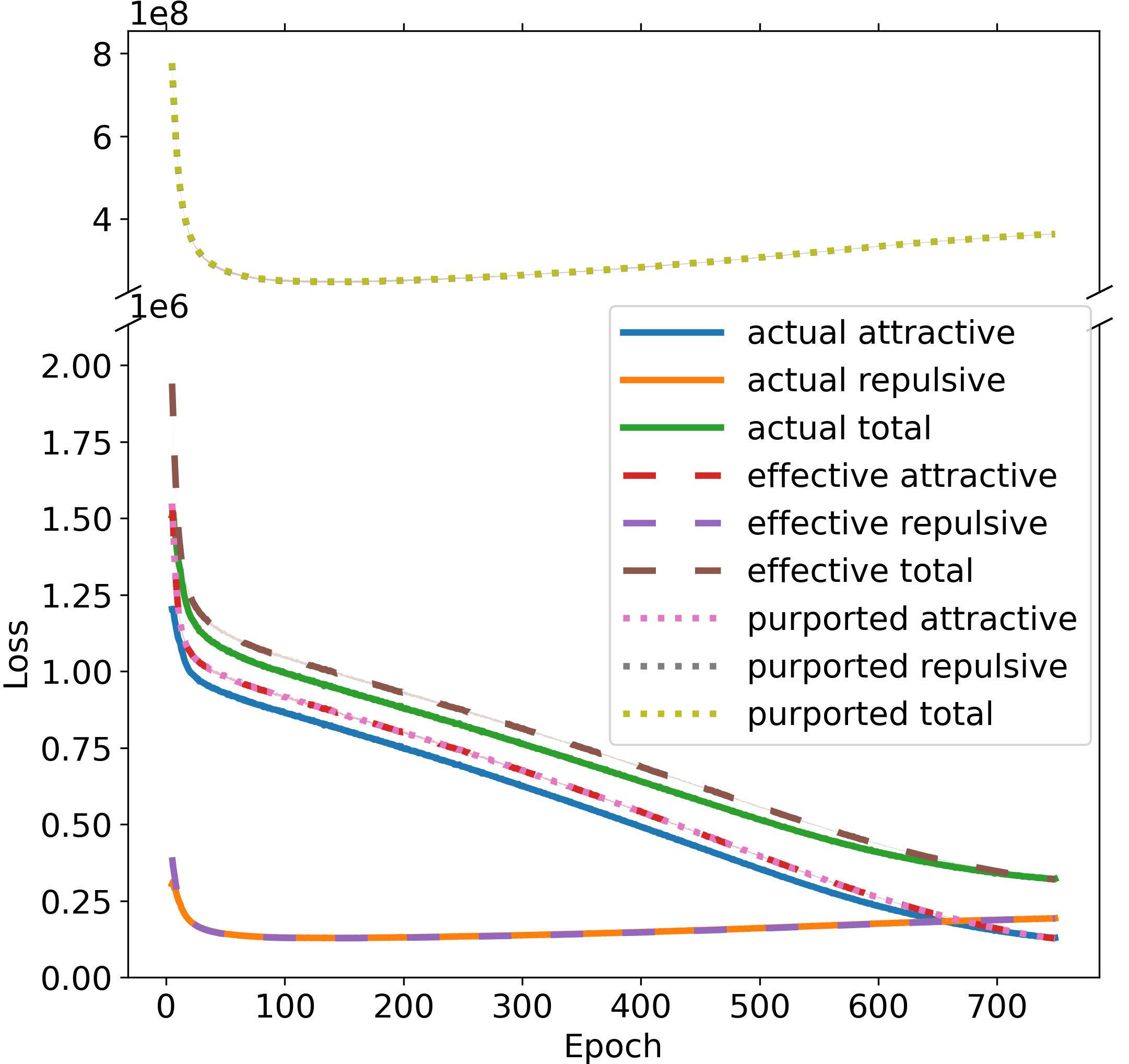}
    \caption{Same as Figure~\ref{fig:c_elegans_losses}, but actual losses are computed with the embeddings at the time of update not with the embeddings after the full epoch as all other losses.}
    \label{fig:c_elegans_losses_during_epoch}
\end{figure}

Our code is publicly available at \url{https://github.com/hci-unihd/UMAPs-true-loss}.

\subsection{Stability}\label{subapp:stability}
Whenever we report loss values, we computed the average over seven runs and give an uncertainty of one standard deviation. Sources of randomness are in the approximate $k$NN computation via nearest neighbor descent~\cite{dong2011efficient}, the Gaussian noise added to the Laplacian Eigenmap initialization and the sampling of the toy ring data itself. Note that the sampling based optimization procedure is implemented deterministically and includes an edge $ij$ every $\max_{ab} \mu_{ab} / \mu_{ij} $-th epoch. We find that the deviation in the loss values is very small across different runs. In fact, as the standard deviation is barely visible in Figure~\ref{fig:c_elegans_losses}, we include the same figure but with shaded areas corresponding to ten standard deviations in Figure~\ref{fig:c_elegans_after_losses_10_std_devs}. Nevertheless, the visual effect of different random seeds can be substantial as depicted in Figure~\ref{fig:c_elegans_random_seeds}.

\subsection{Compute}\label{subapp:compute}
We ran all our experiments on a machine with 20 ``Intel(R) Xeon(R) Silver 4114 CPU @ 2.20GHz" CPUs and six ``Nvidia GeForce GTX 1080 Ti'' GPUs. We only ever used a single GPU and solely for computing the effective and purported losses $\tilde{\mathcal{L}}$ (eq.~\eqref{eq:eff_loss}) and $\mathcal{L}$ (eq.~\eqref{eq:UMAP_obj_embd}). Table~\ref{tab:run_times} shows the run times for the main experiments averaged over 7 runs. Uncertainties indicate one standard deviation.
Logging the losses during optimization quadruples UMAP's run time on the C. elegans dataset. This is due to the quadratic complexity of evaluating the effective and purported loss functions. But with our GPU implementation this longer run time is still easily manageable for the reasonably large real world C. elegans dataset. The toy ring experiments with a dense and thus much larger input graph take about 25 times longer than with the normal, sparse input similarities.

We estimate the total compute by adding the run times of the experiments necessary to reproduce the paper. The number of comparable experiments needed to reproduce the paper is given in Table~\ref{tab:run_times}. The total run time amounts to about 17.5 hours. %

\begin{table}
  \caption{Run times of key experiments averaged over seven runs with standard deviation and number of runs of similar experiments needed to reproduce the paper}
  \label{tab:run_times}
  \centering
  \begin{tabular}{lcccc}
    \toprule
    Experiment & C. elegans w/o  & C. elegans with  & toy ring  & toy ring with  dense input\\
    &loss logging & loss logging (Fig.~\ref{fig:c_elegans_losses}) & (Fig.~\ref{subfig:toy_ring_UMAP})  &   similarities (Fig.~\ref{subfig:toy_ring_dense})\\
    \midrule
    Run time [s] & $511 \pm 6$ & $1995 \pm 5$ & $47.5 \pm 0.14$ &  $1236 \pm 4$\\
    Runs & $7$ %
                   & $23$ %
                   & $12$ %
                   & $10$ %
                   \\ 
    \bottomrule
  \end{tabular}
\end{table}

\begin{figure}
    \centering
    \includegraphics[width=0.5\textwidth]{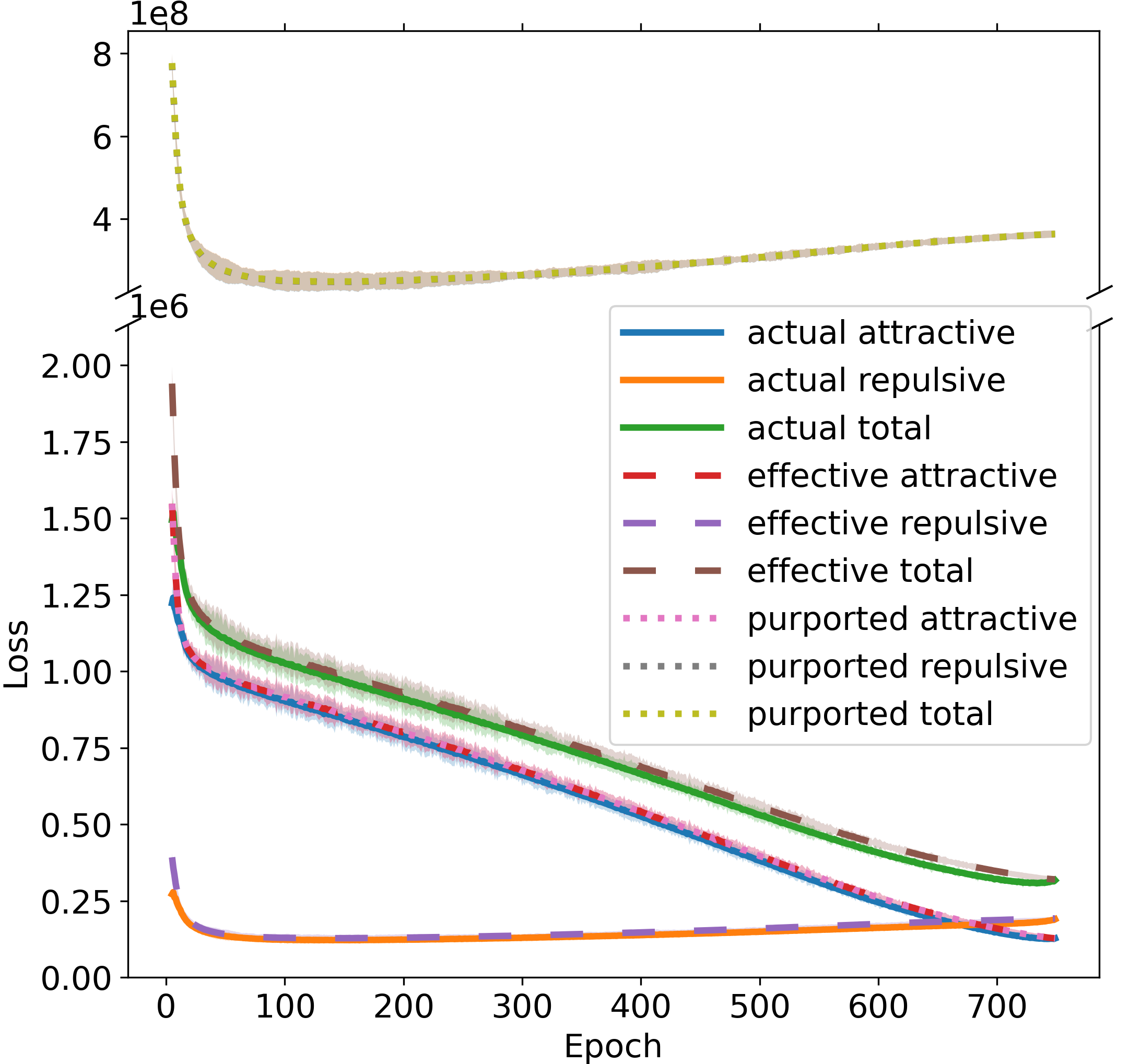}
    \caption{Same as Figure~\ref{fig:c_elegans_losses} but here the shaded region corresponds to ten standard deviations.}
    \label{fig:c_elegans_after_losses_10_std_devs}
\end{figure}

\begin{figure}
    \begin{subfigure}[t]{0.25\textwidth}
        \centering
        \includegraphics[width=.9\linewidth]{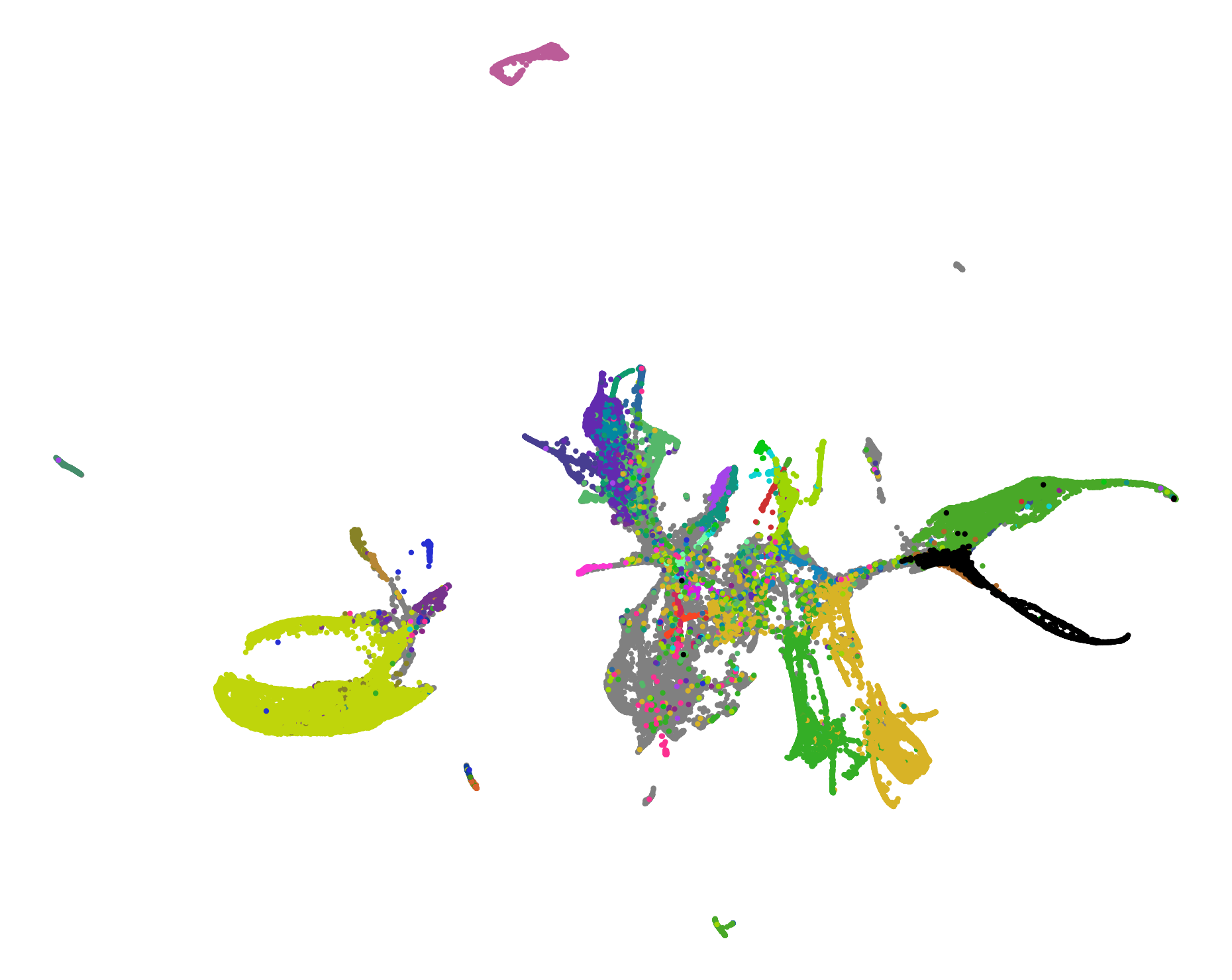}
        \caption{Seed 0}
        \label{fig:c_elegans_UMAP_seed_0}
    \end{subfigure}%
    \begin{subfigure}[t]{0.25\textwidth}
        \centering
        \includegraphics[width=.9\linewidth]{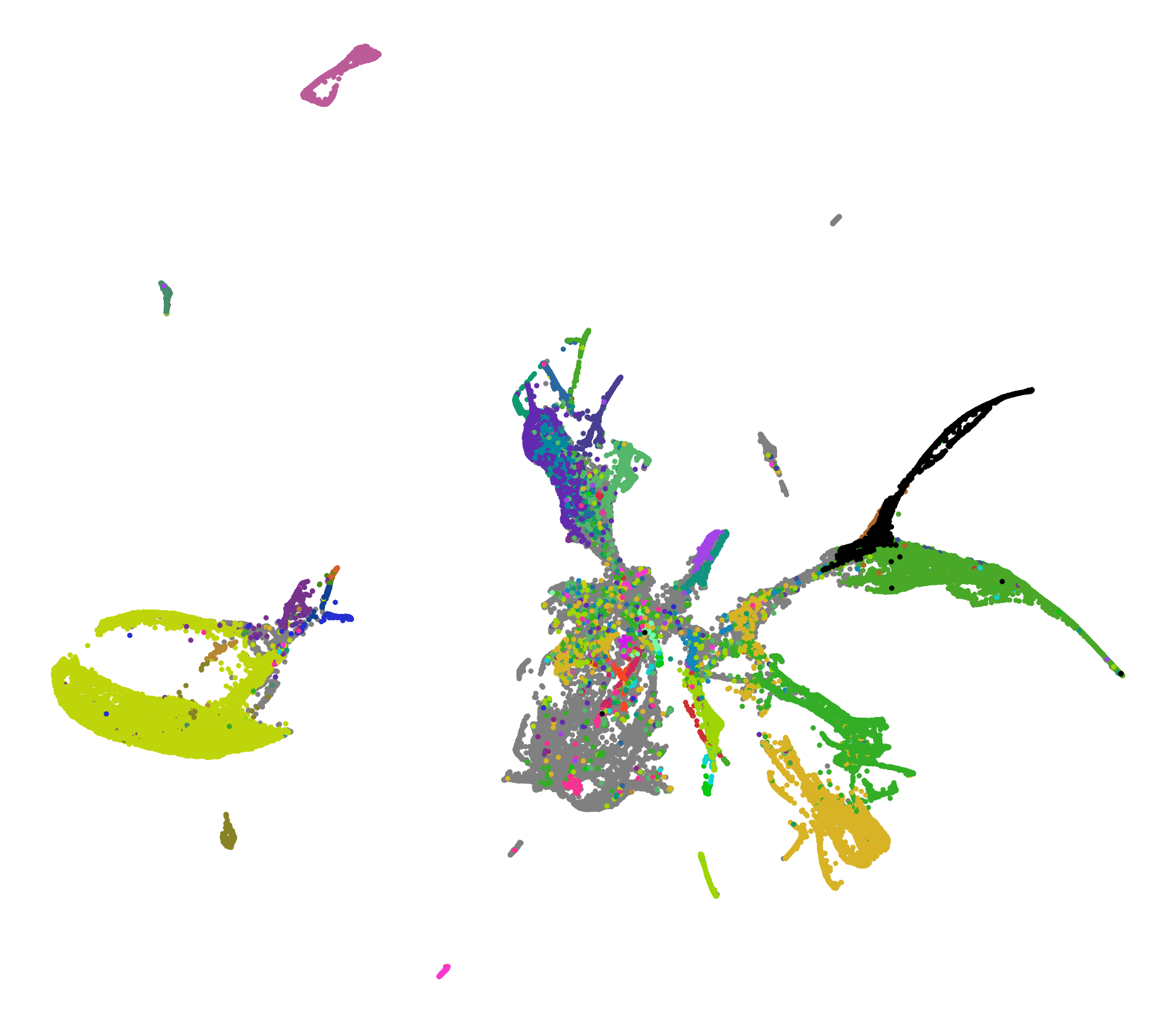}
        \caption{Seed 1}
        \label{fig:c_elegans_UMAP_seed_1}
    \end{subfigure}%
    \begin{subfigure}[t]{0.25\textwidth}
        \centering
        \includegraphics[width=.9\linewidth]{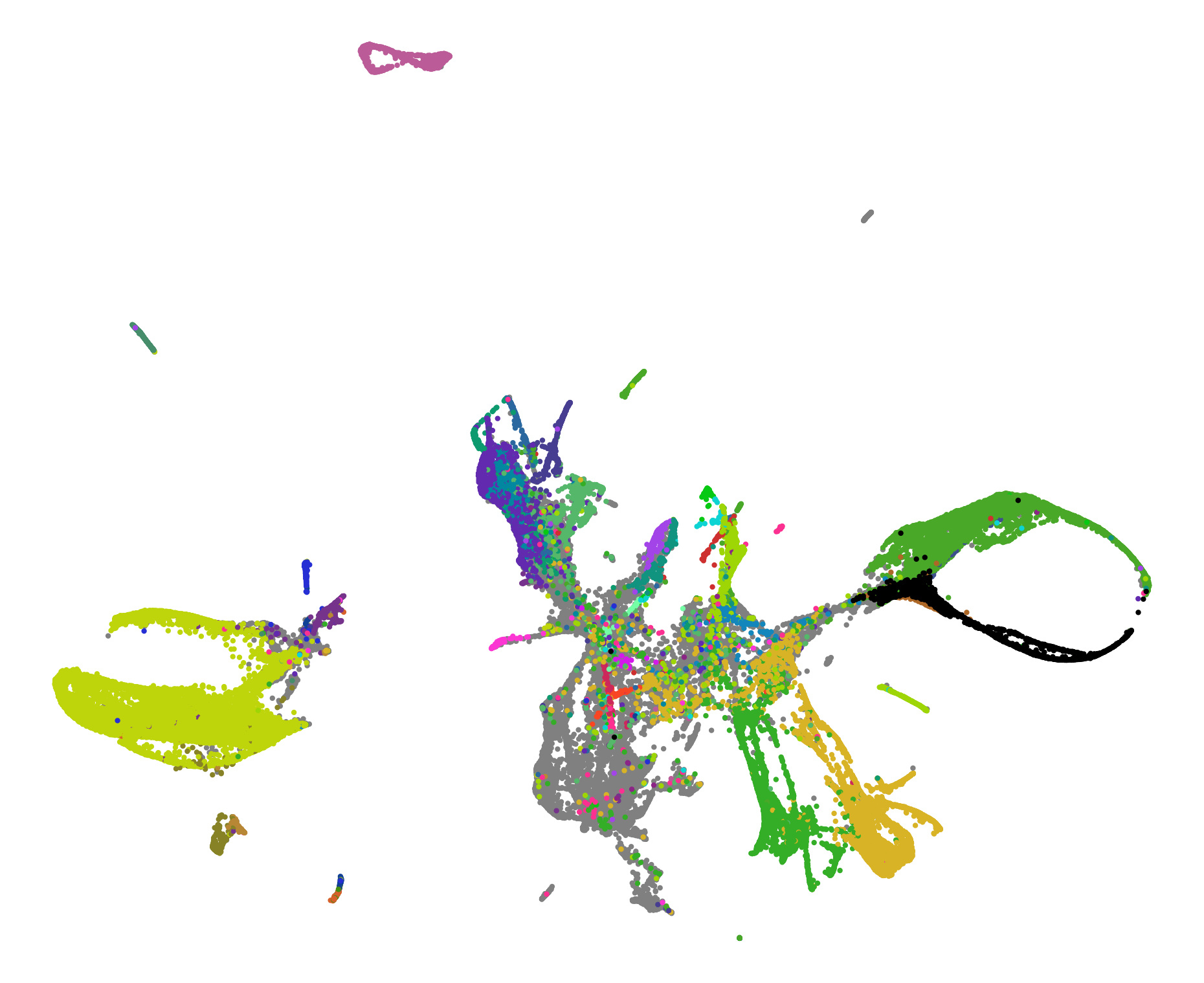}
        \caption{Seed 2}
        \label{fig:c_elegans_UMAP_seed_2}
    \end{subfigure}%
    \begin{subfigure}[t]{0.25\textwidth}
        \centering
        \includegraphics[width=.9\linewidth]{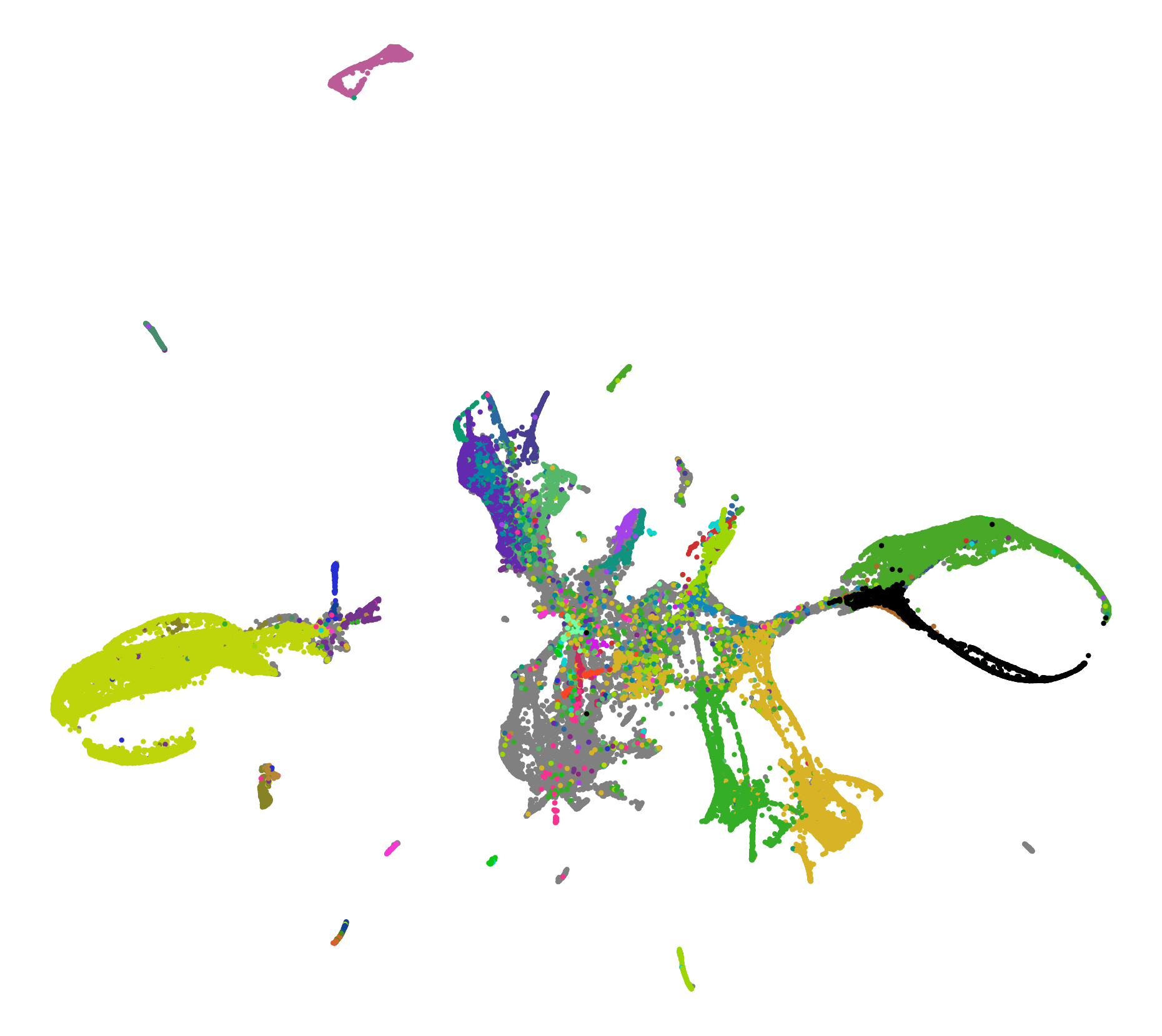}
        \caption{Seed 3}
        \label{fig:c_elegans_UMAP_seed_3}
    \end{subfigure}
    
    \begin{subfigure}[t]{0.25\textwidth}
        \centering
        \includegraphics[width=.9\linewidth]{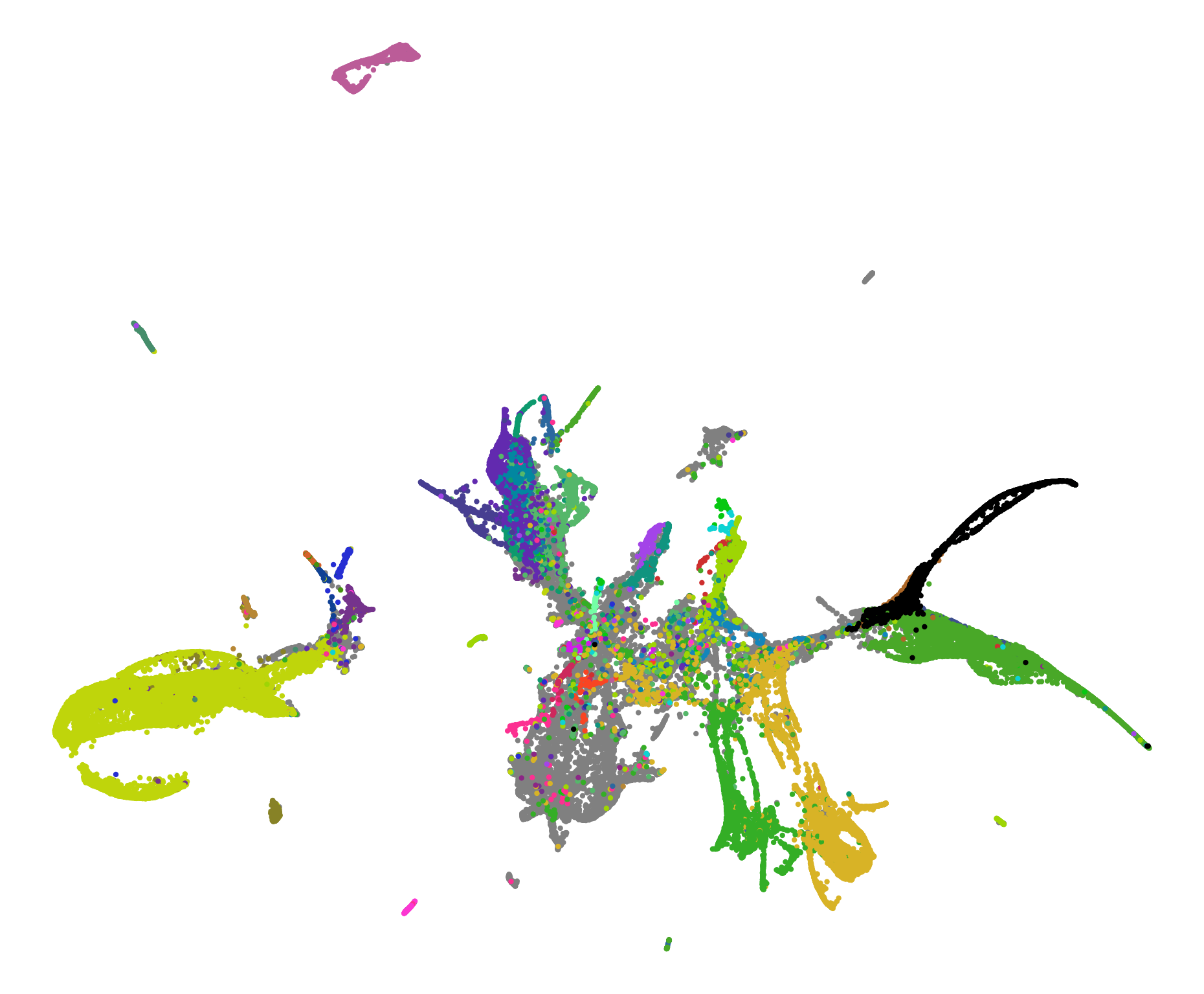}
        \caption{Seed 4}
        \label{fig:c_elegans_UMAP_seed_4}
    \end{subfigure}%
    \begin{subfigure}[t]{0.25\textwidth}
        \centering
        \includegraphics[width=.9\linewidth]{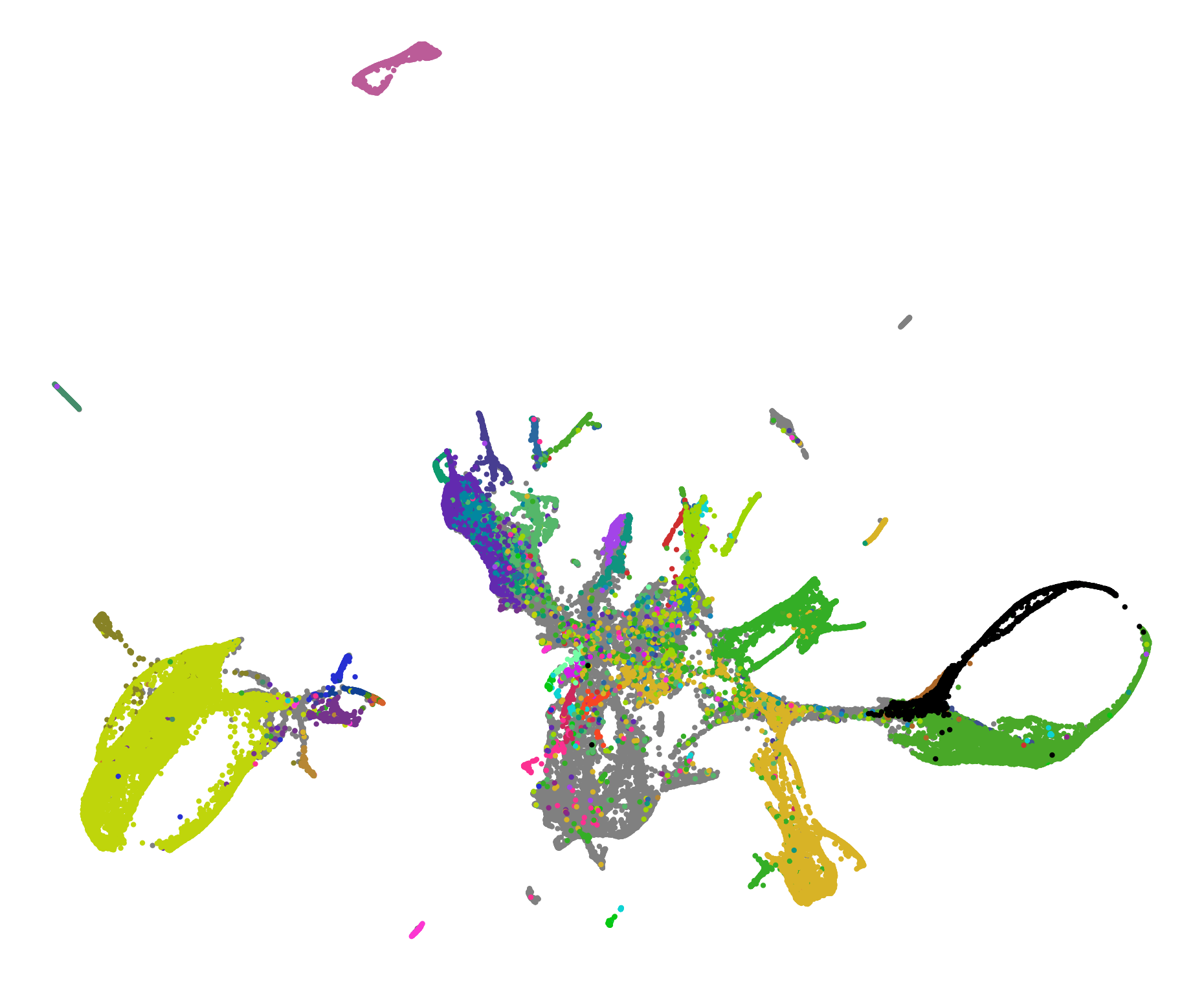}
        \caption{Seed 5}
        \label{fig:c_elegans_UMAP_seed_5}
    \end{subfigure}%
    \begin{subfigure}[t]{0.25\textwidth}
        \centering
        \includegraphics[width=.9\linewidth]{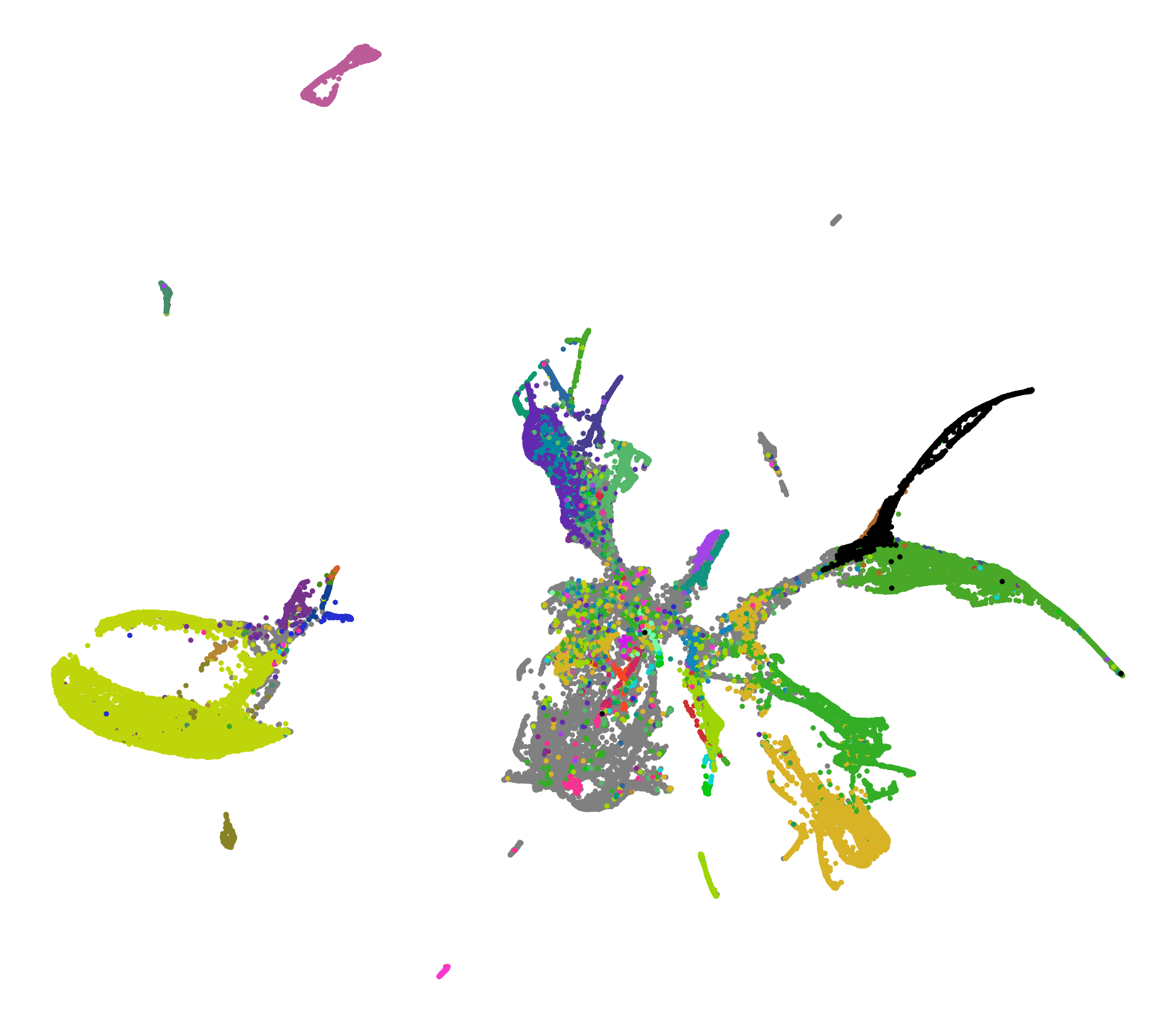}
        \caption{Seed 6}
        \label{fig:c_elegans_UMAP_seed_6}
    \end{subfigure}
    
    \caption{Visualizations of the C. elegans dataset with UMAP and hyperparameters from~\cite{narayan2021assessing} for seven different random seeds. While the losses vary very little, see Figures~\ref{fig:c_elegans_losses} and~\ref{fig:c_elegans_after_losses_10_std_devs}, the visualizations show significant differences, such as closed versus open loops and placement of subgroups. All plots were subjectively flipped and rotated by multiples of $\pi/2$ to ease a visual comparison.}
    \label{fig:c_elegans_random_seeds}
\end{figure}

\section{Societal impact}\label{app:societal_impact}
Together with $t$SNE, UMAP is state-of-the-art for visualizing high-dimensional datasets. It is particularly popular for gene expression data. Our work contributes to the deeper understanding of UMAP, which can benefit many applications, especially in biology. We hope that our contribution will both further theoretical research in non-linear dimension reduction techniques as well as help practitioners interpret their UMAP results more faithfully.

Nevertheless, we neither provide a holistic explanation for UMAP's behavior nor establish faithfulness guarantees. Hence, we caution against overconfidence in UMAP visualizations: Insights gained from exploratory data analysis with UMAP should never be take at face value but experimentally validated.

Exploratory data analysis is a general tool and the societal impact depends on the analyst's intention and the analyzed data. For instance, on data indirectly containing personally identifiable information, UMAP insights might constitute privacy violations. 

Logging UMAP's loss as done in this work increases the computational footprint noticeably. While instructive for validating our theoretical results, we believe that in a typical use case UMAP losses do not need to be logged. We therefore recommend to avoid the additional compute unless fine-grained analysis of UMAP's optimization procedure is needed, for instance to investigate unexpected results.

\section{Additional figures}\label{app:additional_figures}

\begin{figure}[!htb]
    \begin{subfigure}[t]{0.25\textwidth}
        \centering
        \includegraphics[width=.9\linewidth]{figures/toy_circle_1k_4_0_5_original_seed_3.png}
        \caption{Original data}
        \label{subfig:app_toy_ring_original}
    \end{subfigure}%
    \begin{subfigure}[t]{0.25\textwidth}
        \centering
        \includegraphics[width=.9\linewidth]{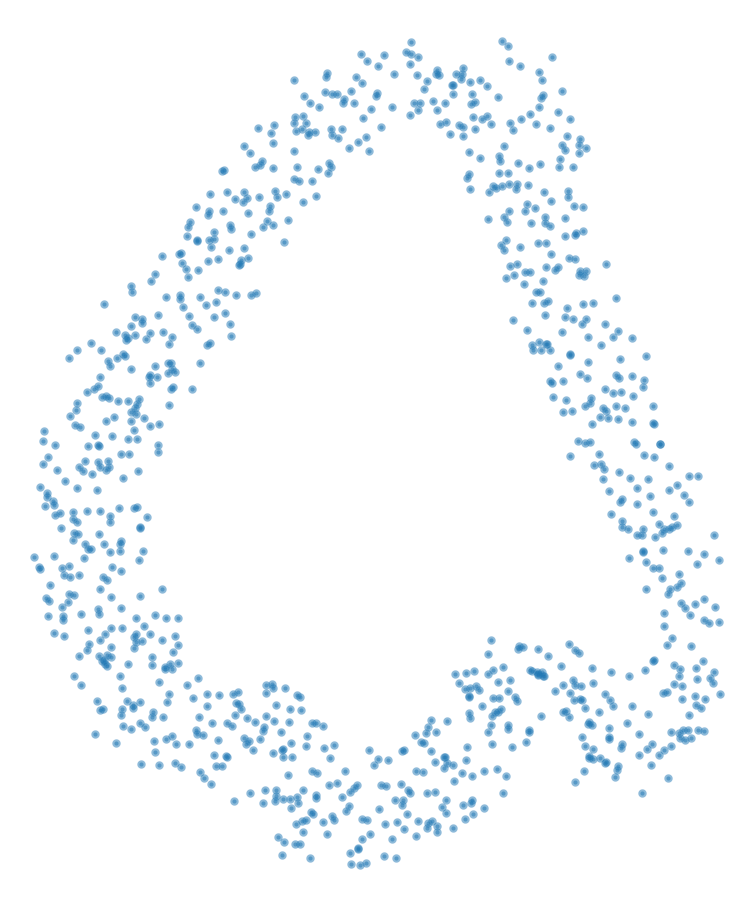}
        \caption{Init data,\\
                 dense similarities}
        \label{subfig:app_toy_ring_init_graph}
    \end{subfigure}%
    \begin{subfigure}[t]{0.25\textwidth}
        \centering
        \includegraphics[width=.9\linewidth]{figures/toy_circle_1k_4_0_5_init_graph_10000_seed_3.png}
        \caption{Init data,\\
                 dense similarities, \\
                 10000 epochs}
        \label{subfig:app_toy_ring_init_graph_10000}
    \end{subfigure}
    
    \begin{subfigure}[t]{0.25\textwidth}
        \centering
        \includegraphics[width=.9\linewidth]{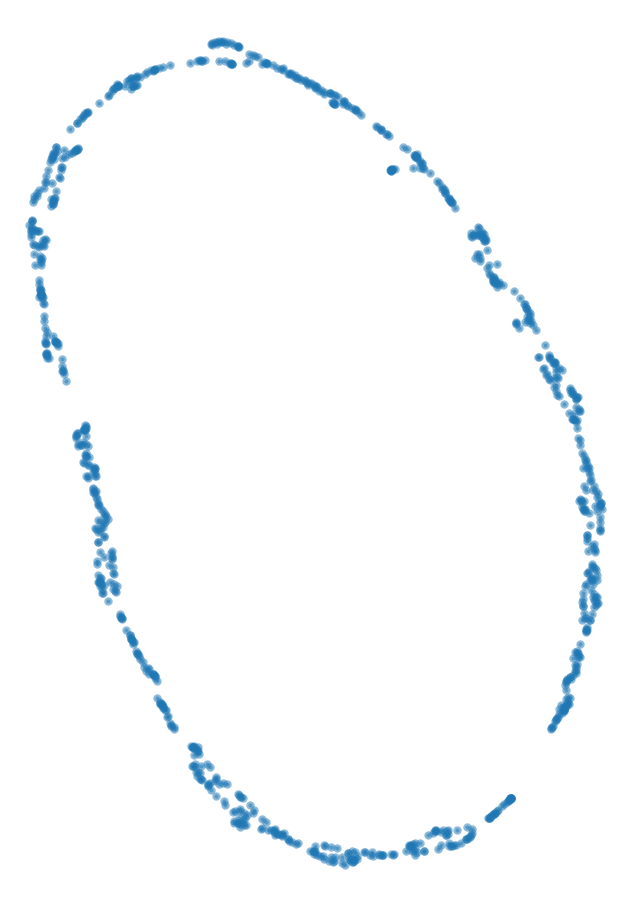}
        \caption{Default UMAP}
        \label{subfig:app_toy_ring_default}
    \end{subfigure}%
    \begin{subfigure}[t]{0.25\textwidth}
        \centering
        \includegraphics[width=.9\linewidth]{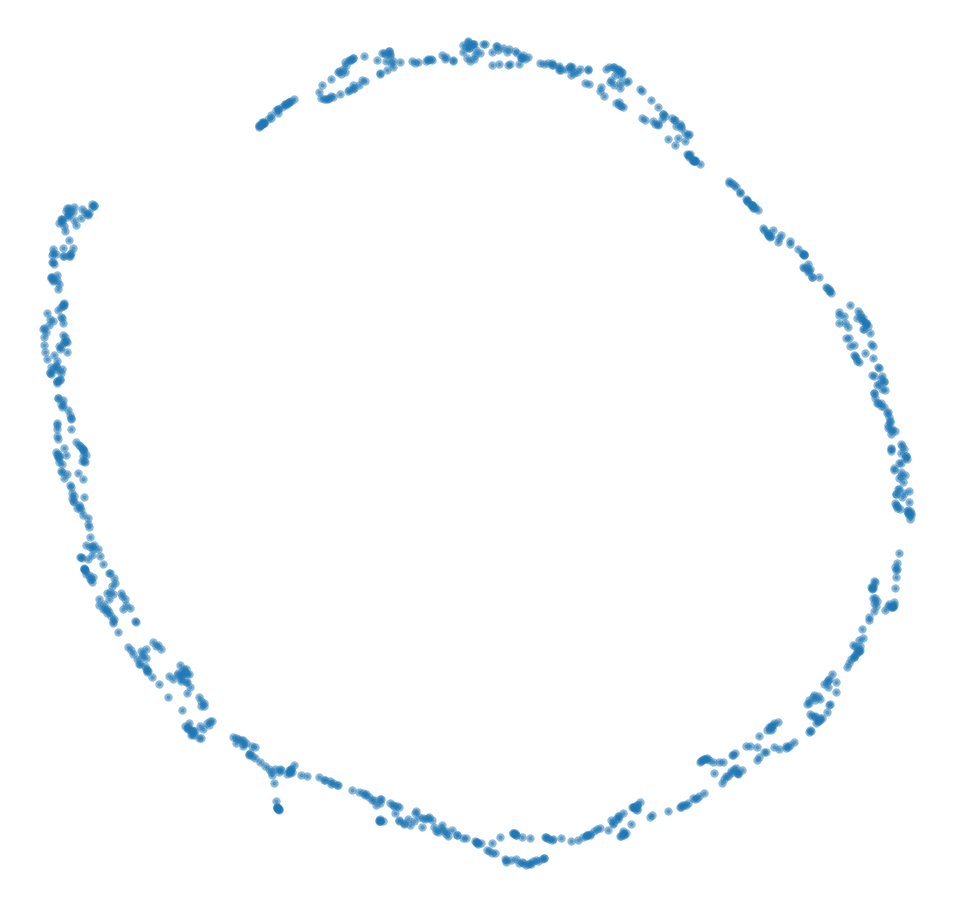}
        \caption{10000 epochs}
        \label{subfig:app_toy_ring_10000}
    \end{subfigure}%
    \begin{subfigure}[t]{0.25\textwidth}
        \centering
        \includegraphics[width=.9\linewidth]{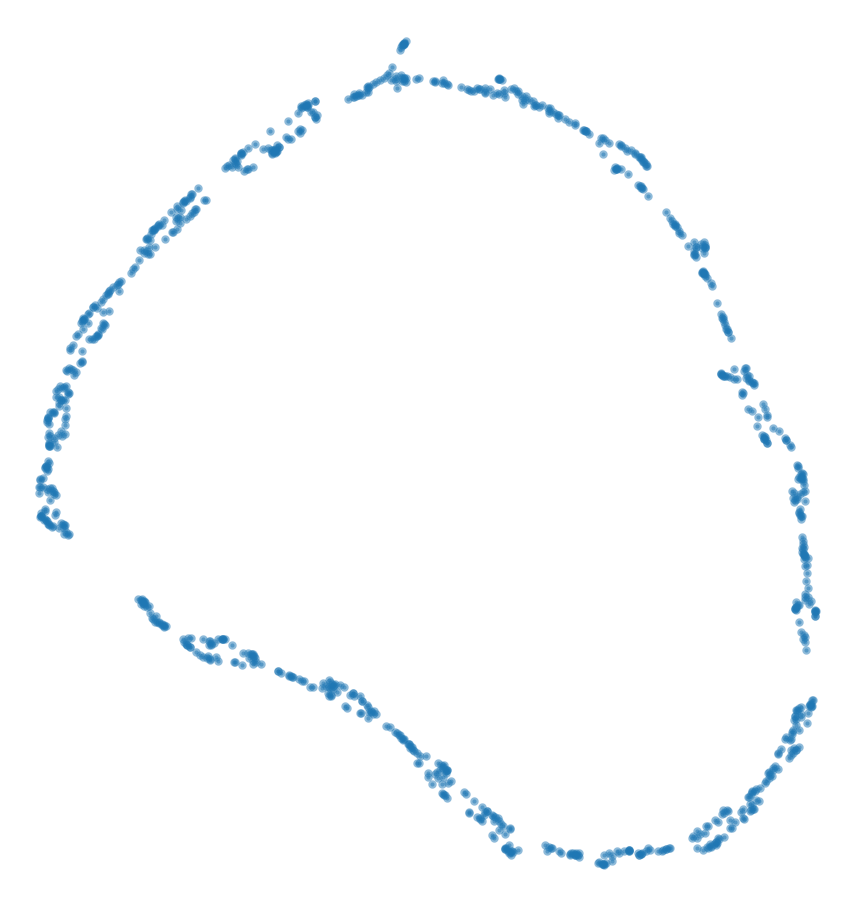}
        \caption{Init data}
        \label{subfig:app_toy_ring_init}
    \end{subfigure}%
    \begin{subfigure}[t]{0.25\textwidth}
        \centering
        \includegraphics[width=.9\linewidth]{figures/toy_circle_1k_4_0_5_init_10000_seed_3.png}
        \caption{Init data,\\
                 10000 epochs}
        \label{subfig:app_toy_ring_init_10000}
    \end{subfigure}%
    \caption{UMAP does not preserve the data even when embedding to the input dimension. Extension of Figure~\ref{fig:toy_ring}. \ref{subfig:app_toy_ring_original}~Original data: 1000 uniform samples from ring in 2D. \ref{subfig:app_toy_ring_init_graph}~Result of UMAP when initialized with the original data and using  dense input space similarities computed from the original data with $\phi$. \ref{subfig:app_toy_ring_init_graph_10000} Same as \ref{subfig:app_toy_ring_init_graph} but optimized for 10000 epochs. \ref{subfig:app_toy_ring_default} UMAP visualization with default hyperparameters. \ref{subfig:app_toy_ring_10000}~Same as \ref{subfig:app_toy_ring_default} but optimized for 10000 epochs. \ref{subfig:app_toy_ring_init} UMAP visualization initialized with the original data. \ref{subfig:app_toy_ring_init_10000} Same as \ref{subfig:app_toy_ring_init} but optimized for 10000 epochs.}
    \label{fig:add_toy_ring}
\end{figure}

\begin{figure}[tb]
    \begin{subfigure}[t]{0.33\textwidth}
        \centering
        \includegraphics[width=.9\linewidth]{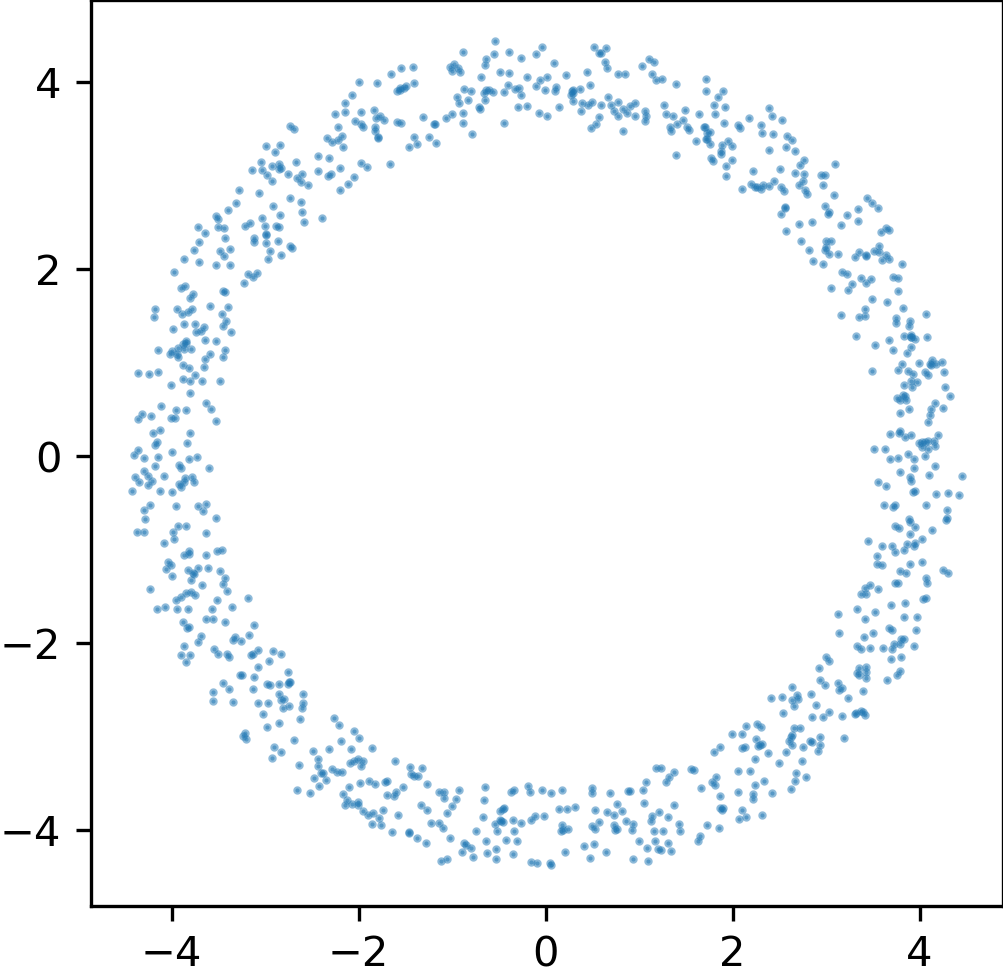}
        \caption{Original data}
        \label{subfig:toy_ring_original_push_tail}
    \end{subfigure}%
    \begin{subfigure}[t]{0.33\textwidth}
        \centering
        \includegraphics[width=.9\linewidth]{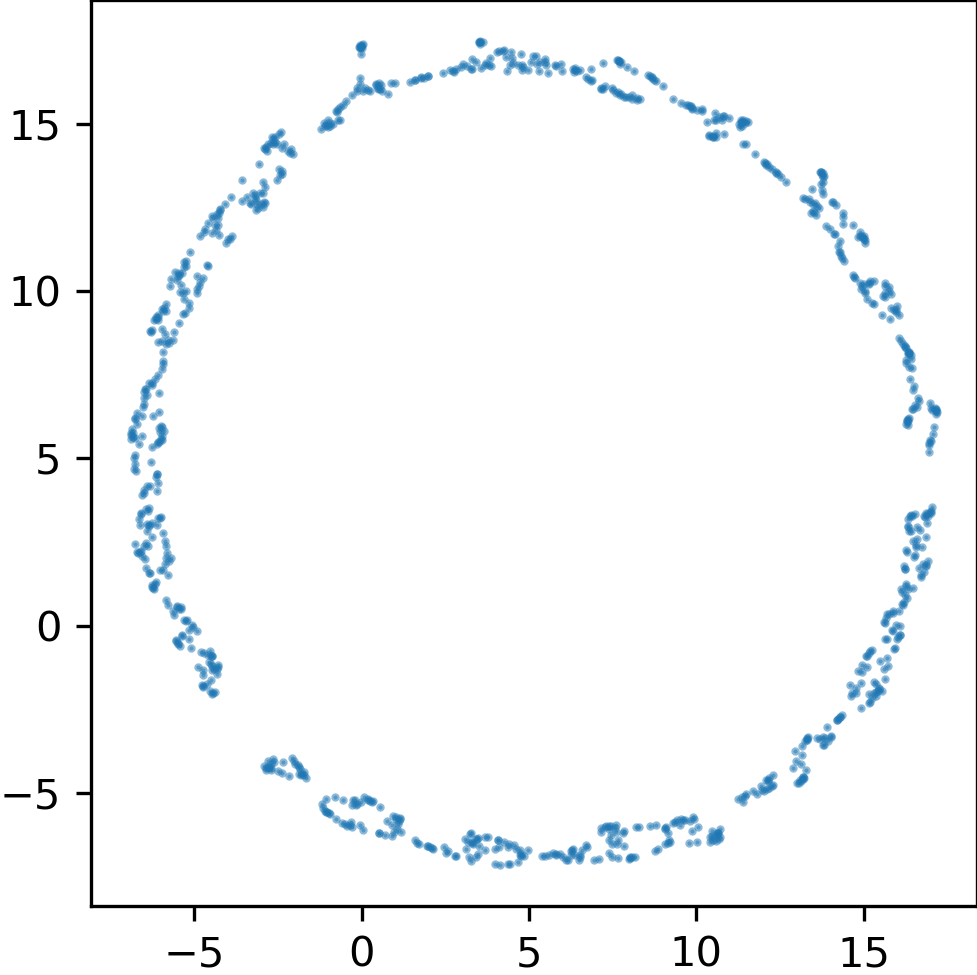}
        \caption{UMAP}
        \label{subfig:toy_ring_UMAP_push_tail}
    \end{subfigure}%
    \begin{subfigure}[t]{0.33\textwidth}
        \centering
        \includegraphics[width=.9\linewidth]{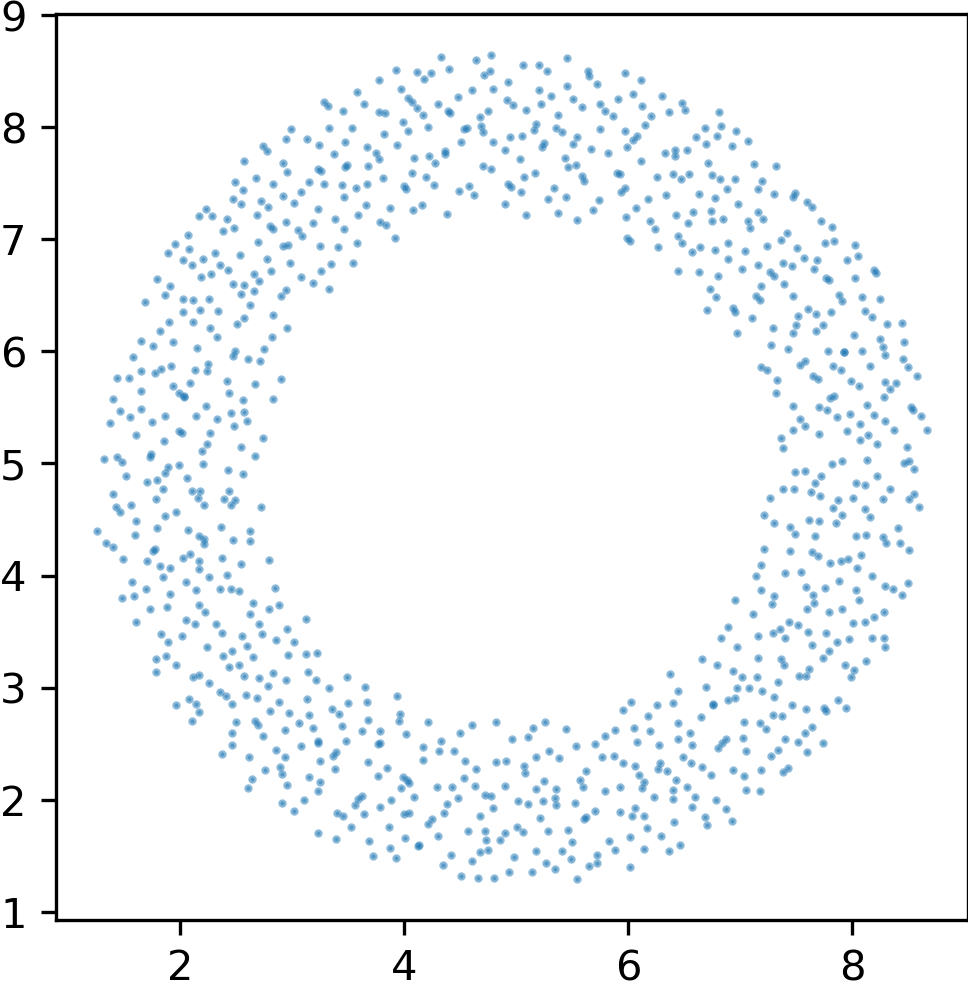}
        \caption{UMAP from dense similarities}
        \label{subfig:toy_ring_dense_push_tail}
    \end{subfigure}
    \caption{Same as figure~\ref{fig:toy_ring} but here the tail of a negative sample is repelled from its head. \ref{subfig:toy_ring_UMAP_push_tail} looks similarly over-contracted but slightly rounder than~\ref{subfig:toy_ring_UMAP}. \ref{subfig:toy_ring_dense_push_tail} shows wider than expected ring structure similar to~\ref{subfig:toy_ring_dense} but without the spurious curves. Instead the radius of the ring is smaller than in the original. Both the larger ring width and the smaller radius match the analysis in Section~\ref{subsec:interpret_toy_ring}.}
    \label{fig:toy_ring_push_tail}
\end{figure}

\begin{figure}[htb]
    \begin{subfigure}[t]{0.33\textwidth}
        \centering
        \includegraphics[width=.9\linewidth]{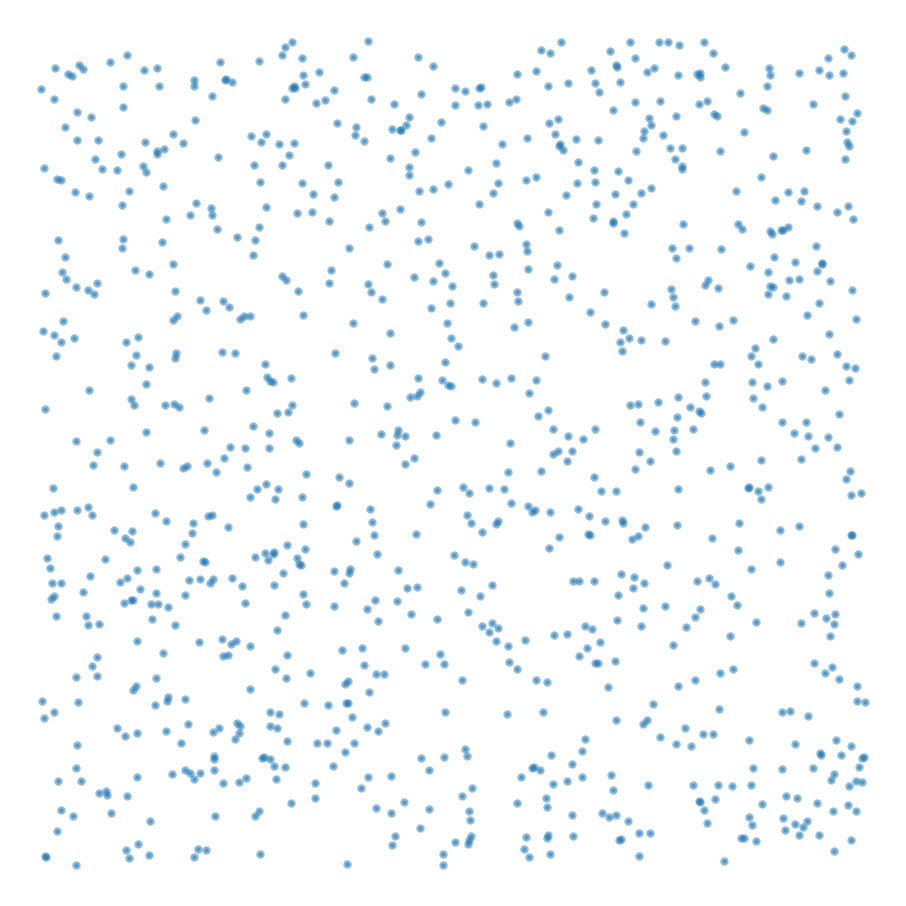}
        \caption{Original data}
        \label{subfig:toy_uniform_original}
    \end{subfigure}%
    \begin{subfigure}[t]{0.33\textwidth}
        \centering
        \includegraphics[width=.9\linewidth]{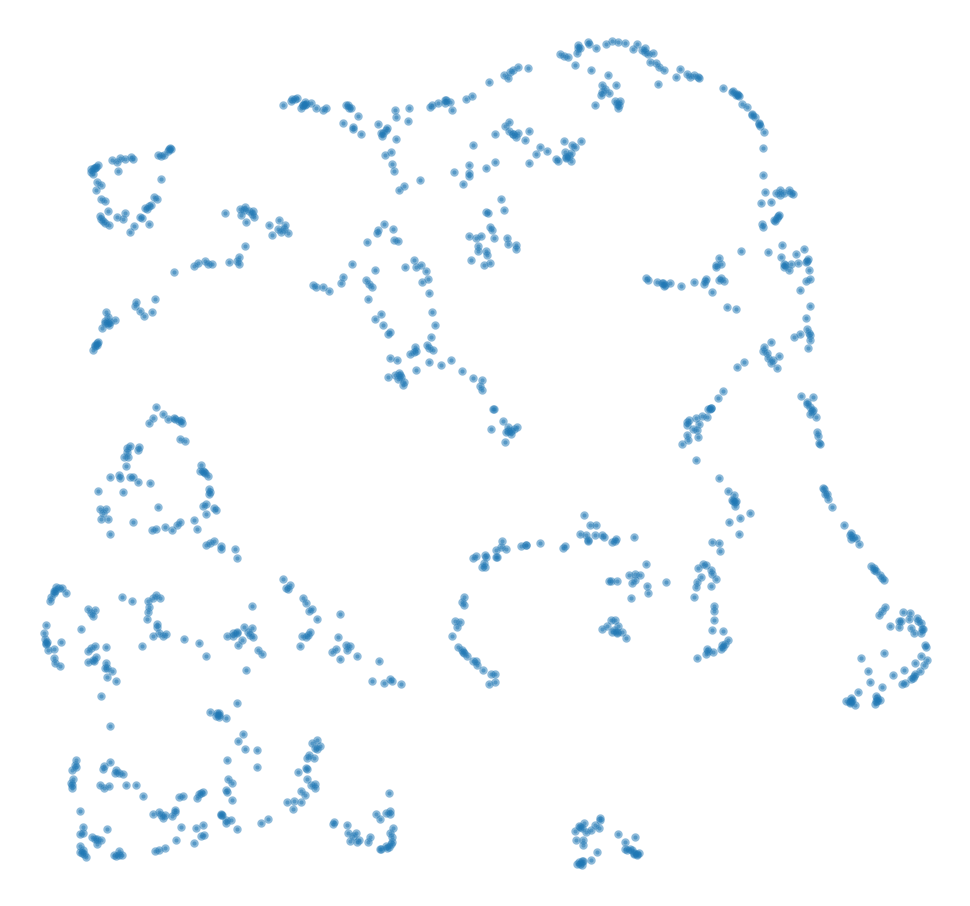}
        \caption{UMAP}
        \label{subfig:toy_uniform_UMAP}
    \end{subfigure}%
    \begin{subfigure}[t]{0.33\textwidth}
        \centering
        \includegraphics[width=.9\linewidth]{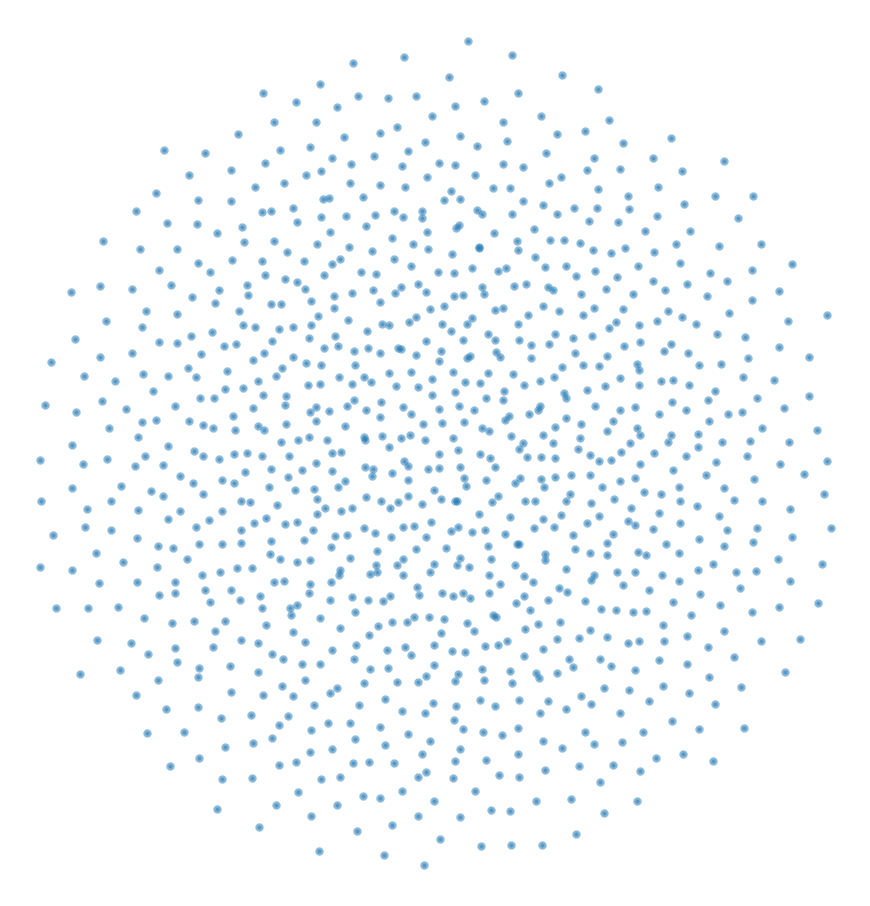}
        \caption{UMAP from dense similarities}
        \label{subfig:toy_uniform_dense}
    \end{subfigure}
    \caption{UMAP does not preserve the data even when no dimension reduction is required. \ref{subfig:toy_uniform_original}~Original data consisting of 1000 uniform samples from a unit square in 2D. \ref{subfig:toy_uniform_UMAP}~Result of UMAP after 10000 epochs, initialized with the original data. The embedding is much more clustered than the original data. \ref{subfig:toy_uniform_dense}~Result of UMAP after 10000 epochs for dense input space similarities computed from the original data with $\phi$, initialized with the original embedding. No change would be optimal in this setting. Instead the output is circular with slightly higher density in the middle. It appears even more regular than the original data.}
    \label{fig:toy_uniform}
\end{figure}

\begin{figure}
    \begin{subfigure}[b]{0.45\textwidth}
        \centering
        \includegraphics[width=0.9\linewidth]{figures/c_elegans_densmap_no_leg_seed_0.png}
        \caption{UMAP}
        \label{subfig:c_elegans_UMAP_with_leg}
    \end{subfigure}%
    \begin{subfigure}[b]{0.45\textwidth}
        \centering
        \includegraphics[width=\linewidth]{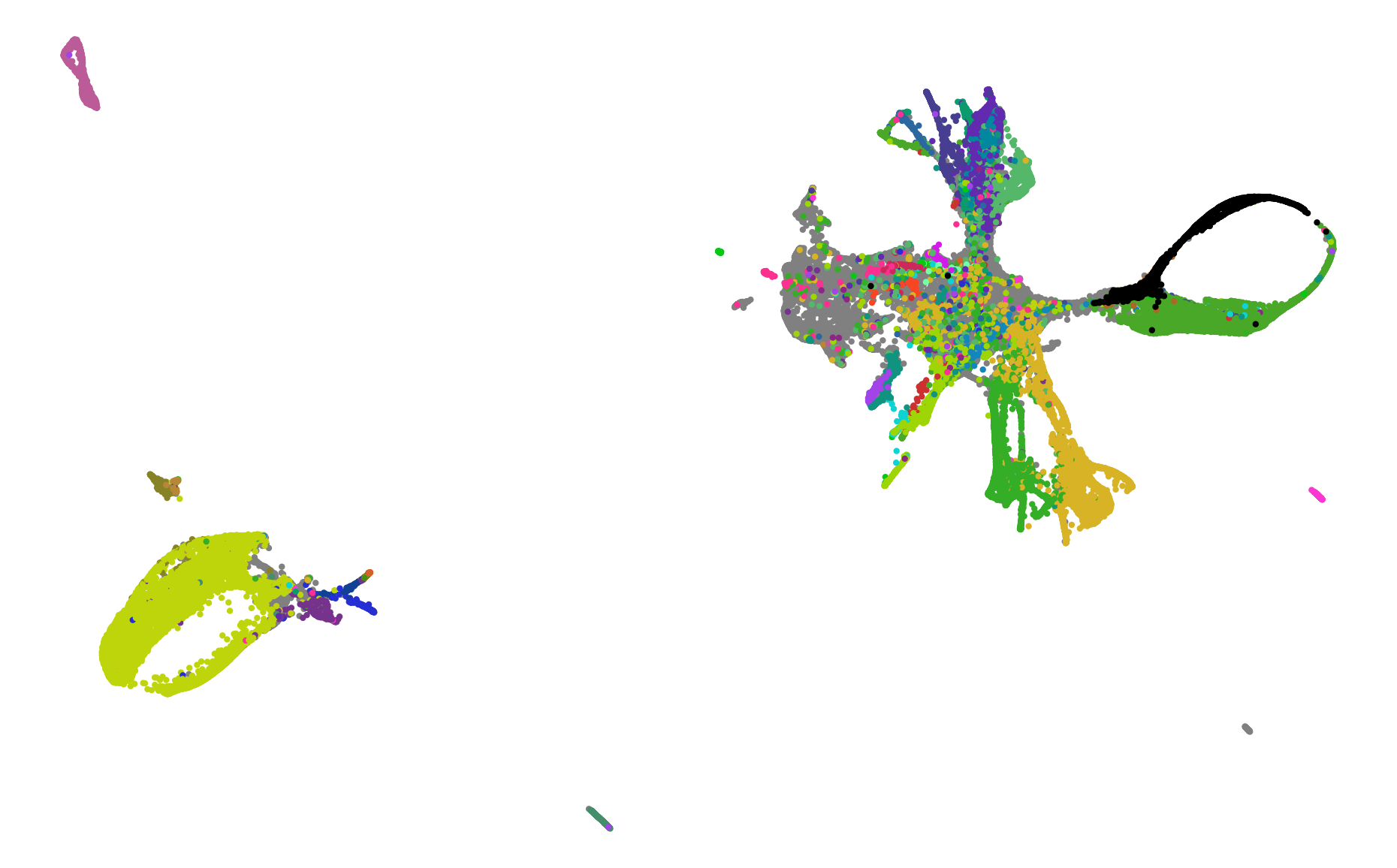}
        \caption{Shared kNN}
        \label{subfig:c_elegans_kNN}
    \end{subfigure}
    
    \begin{subfigure}[b]{0.33\textwidth}
        \centering
        \includegraphics[width=0.9\linewidth]{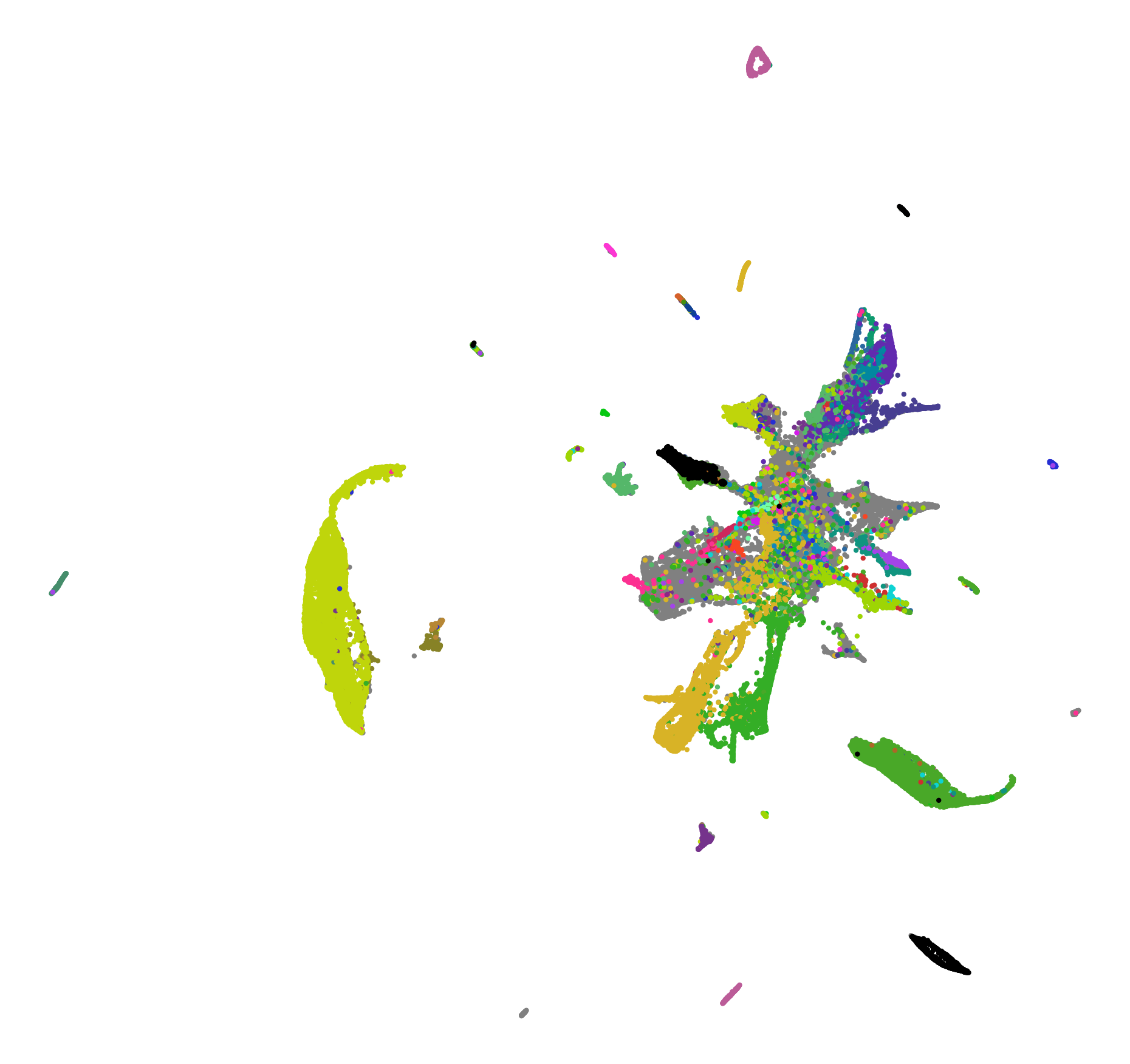}
        \caption{Permuted}
        \label{subfig:c_elegans_perm}
    \end{subfigure}%
    \begin{subfigure}[b]{0.3\textwidth}
        \centering
        \includegraphics[width=0.9\linewidth]{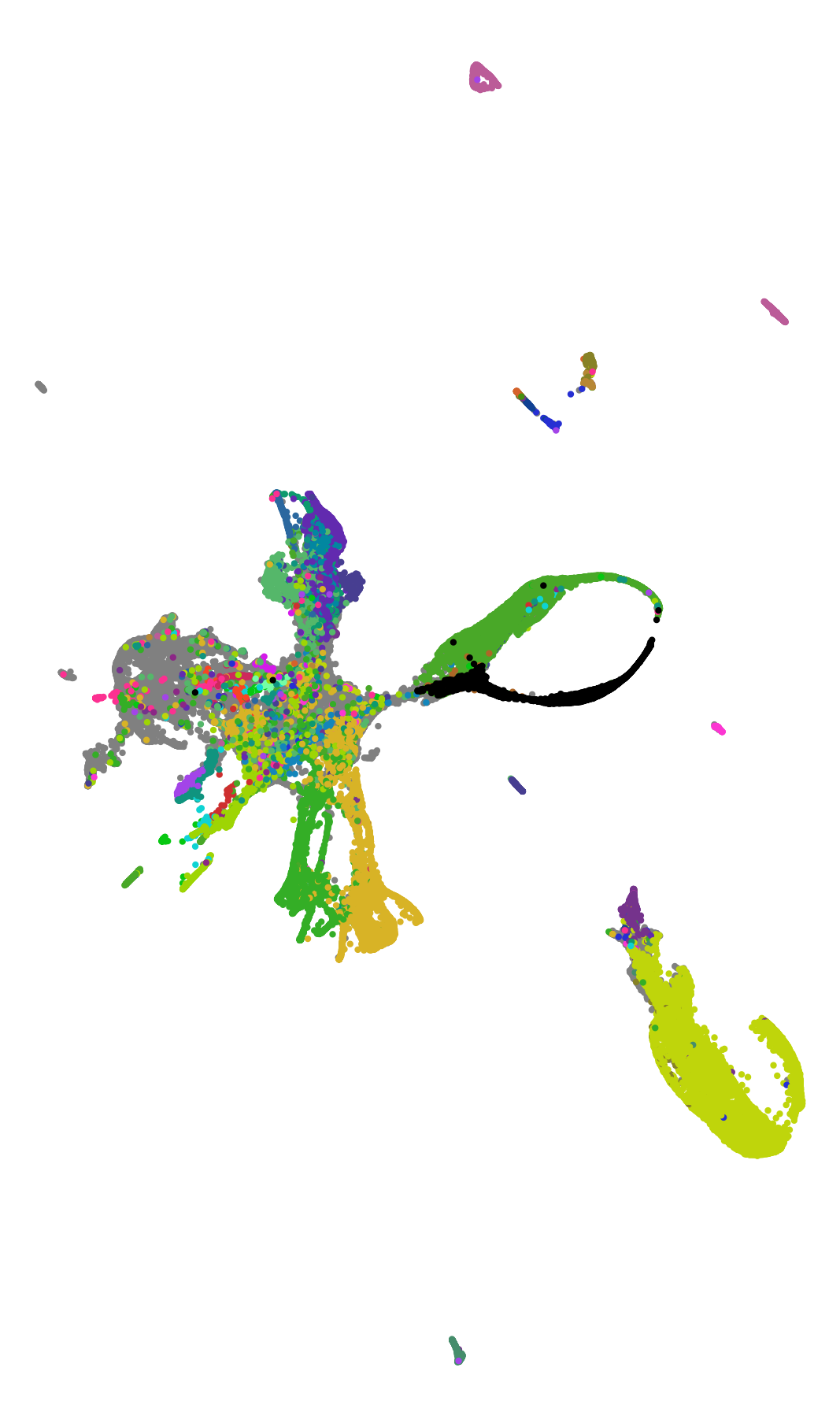}
        \caption{Uniformly random}
        \label{subfig:c_elegans_uni}
    \end{subfigure}%
    \begin{subfigure}[b]{0.33\textwidth}
        \centering
        \includegraphics[width=0.9\linewidth]{figures/c_elegans_densmap_inv_seed_0.png}
        \caption{Inverted}
        \label{subfig:c_elegans_inv_app}
    \end{subfigure}
    
    \begin{subfigure}[b]{\textwidth}
        \centering
        \includegraphics[width=0.9\linewidth]{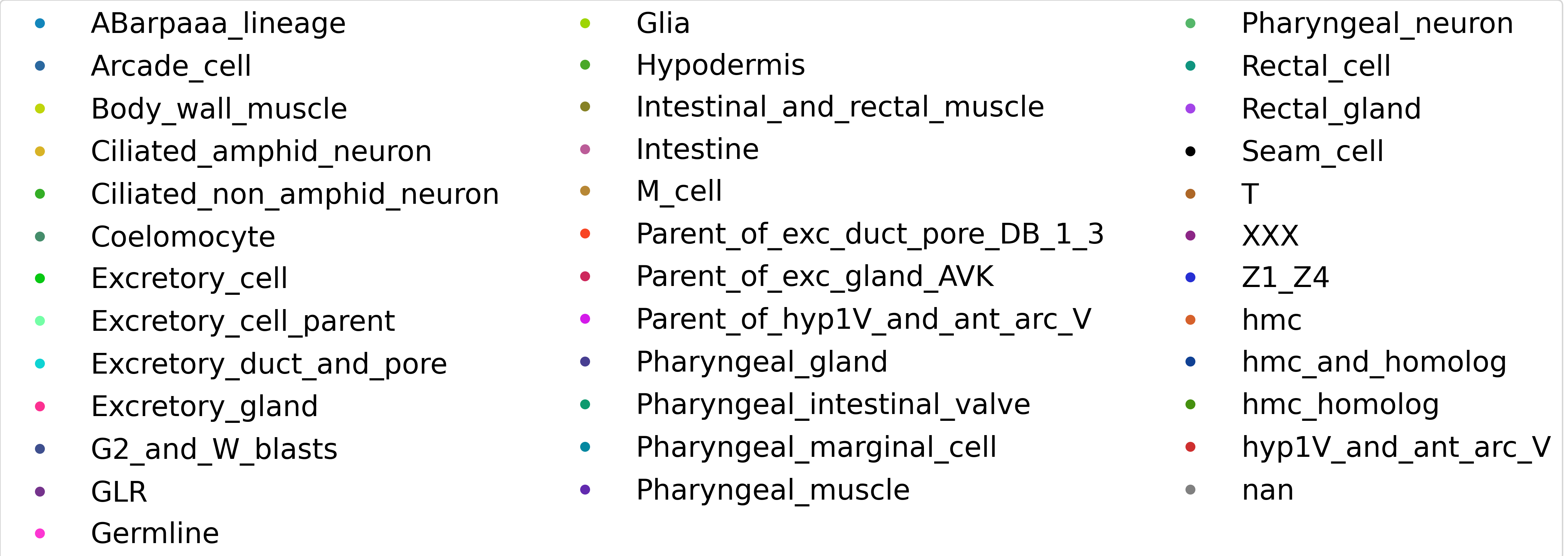}
    \end{subfigure}
\end{figure}
\clearpage
\begin{minipage}{\linewidth}
\captionsetup{type=figure}
    \caption[]{ The precise value of the positive $\mu_{ij}$'s matters little: UMAP produces qualitatively similar results even for severely perturbed $\mu_{ij}$. The panels depict UMAP visualizations based on the hyperparameters in~\cite{narayan2021assessing} but with disturbed positive high-dimensional similarities. \ref{subfig:c_elegans_UMAP_with_leg}: Usual UMAP $\mu_{ij}$'s. 
    \ref{subfig:c_elegans_kNN}: Positive $\mu_{ij}$ all set to one, so that the weights encode the shared $k$NN graph as done in~\cite{bohm2020unifying}.
    \ref{subfig:c_elegans_perm}: Positive $\mu_{ij}$ randomly permuted.
    \ref{subfig:c_elegans_uni}: Positive $\mu_{ij}$ overwritten by uniform random samples from $[0,1]$. 
    \ref{subfig:c_elegans_inv_app}: Positive $\mu_{ij}$ filtered as in UMAP's optimization procedure (set all weights to zero below $\max \mu_{ij} / \text{n\_epochs}$) and inverted at the minimal positive value $\mu_{ij} = \min_{ab} \mu_{ab} / \mu_{ij}$. 
    Amazingly, the visualizations still show the main structures identified by the unimpaired
    UMAP. While \ref{subfig:c_elegans_perm} tears up the seam cells, \ref{subfig:c_elegans_inv_app} even places the outliers conveniently compactly around the main structure. The level of similarity seems particularly high when compared Figure~\ref{fig:c_elegans_random_seeds} which shows that the global placement of subgroups as well as whether loops are open or closed (e.g. seam cells and hypdoermis cells) depends even on the random seed. We used random seed $0$ in this figure. All C. elegans UMAP embeddings were subjectively flipped and rotated by multiples of $\pi/2$ to ease a visual comparison.}
    \label{fig:c_elegans_perturbed}
\end{minipage}

\begin{figure}[ht]
    \begin{subfigure}[b]{0.33\textwidth}
        \centering
        \includegraphics[width=0.9\linewidth]{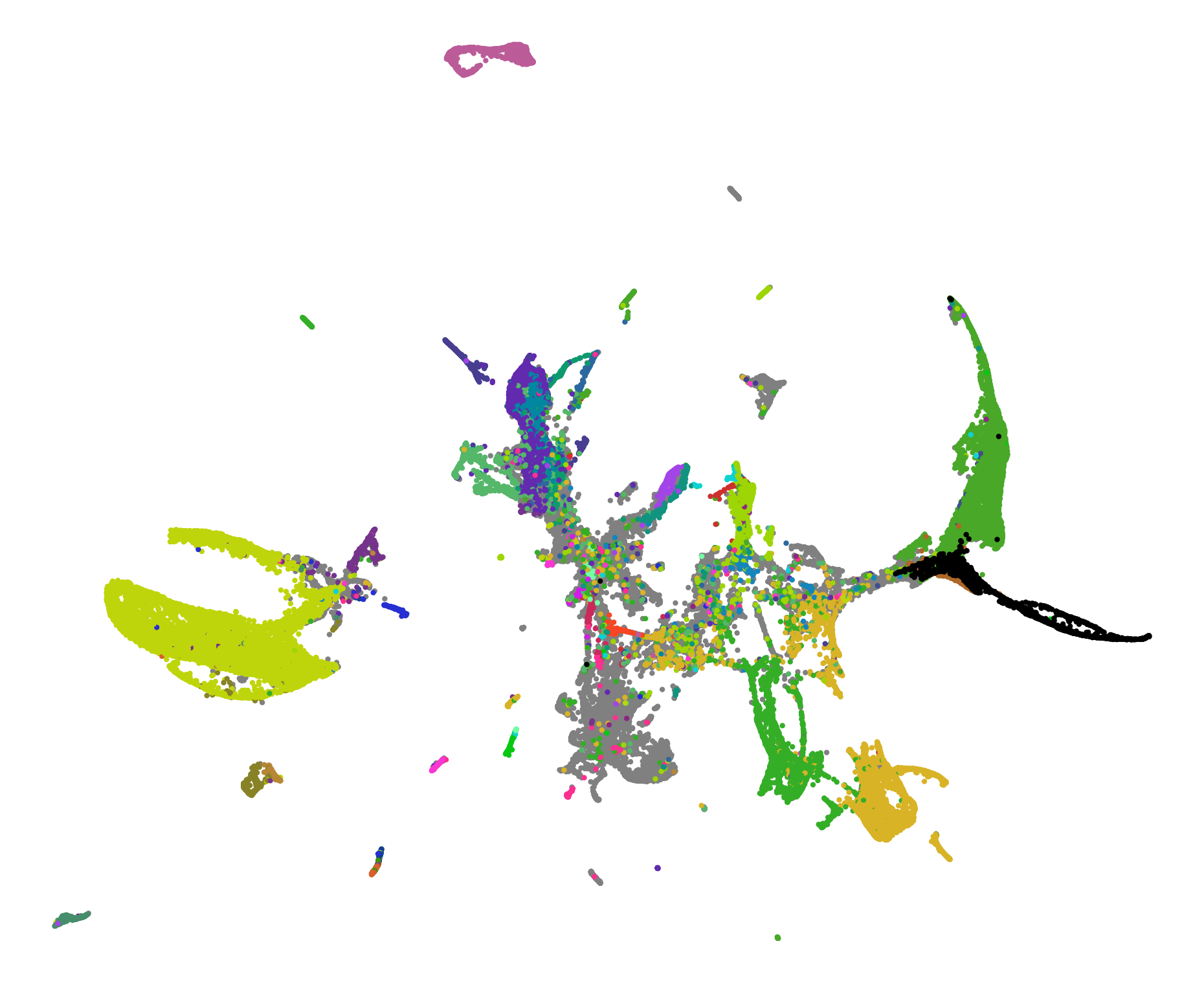}
        \caption{UMAP}
        \label{subfig:c_elegans_UMAP_push_tail}
    \end{subfigure}%
    \begin{subfigure}[b]{0.3\textwidth}
        \centering
        \includegraphics[width=0.9\linewidth]{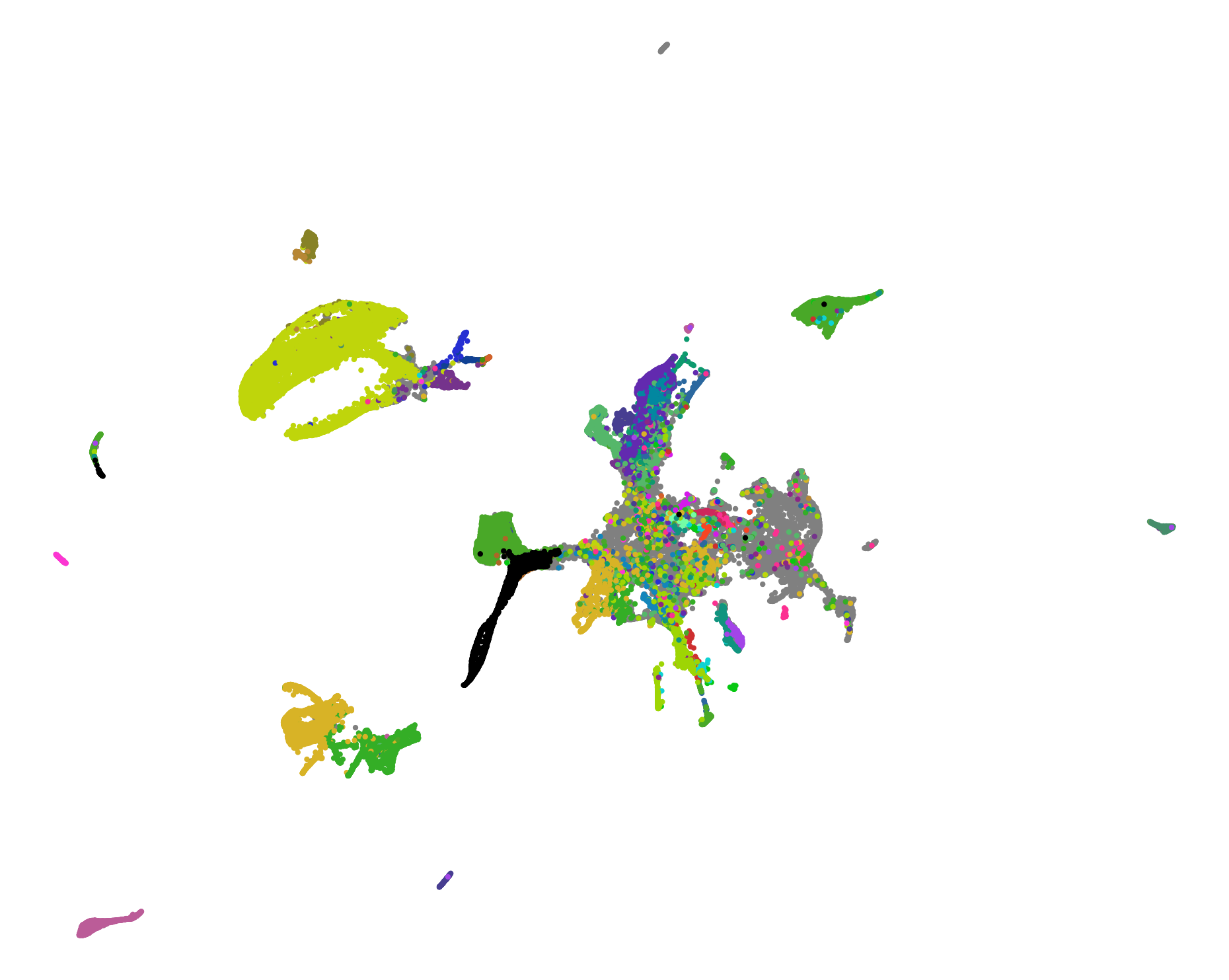}
        \caption{Uniformly random}
        \label{subfig:c_elegans_uni_push_tail}
    \end{subfigure}%
    \begin{subfigure}[b]{0.3\textwidth}
        \centering
        \includegraphics[width=\linewidth]{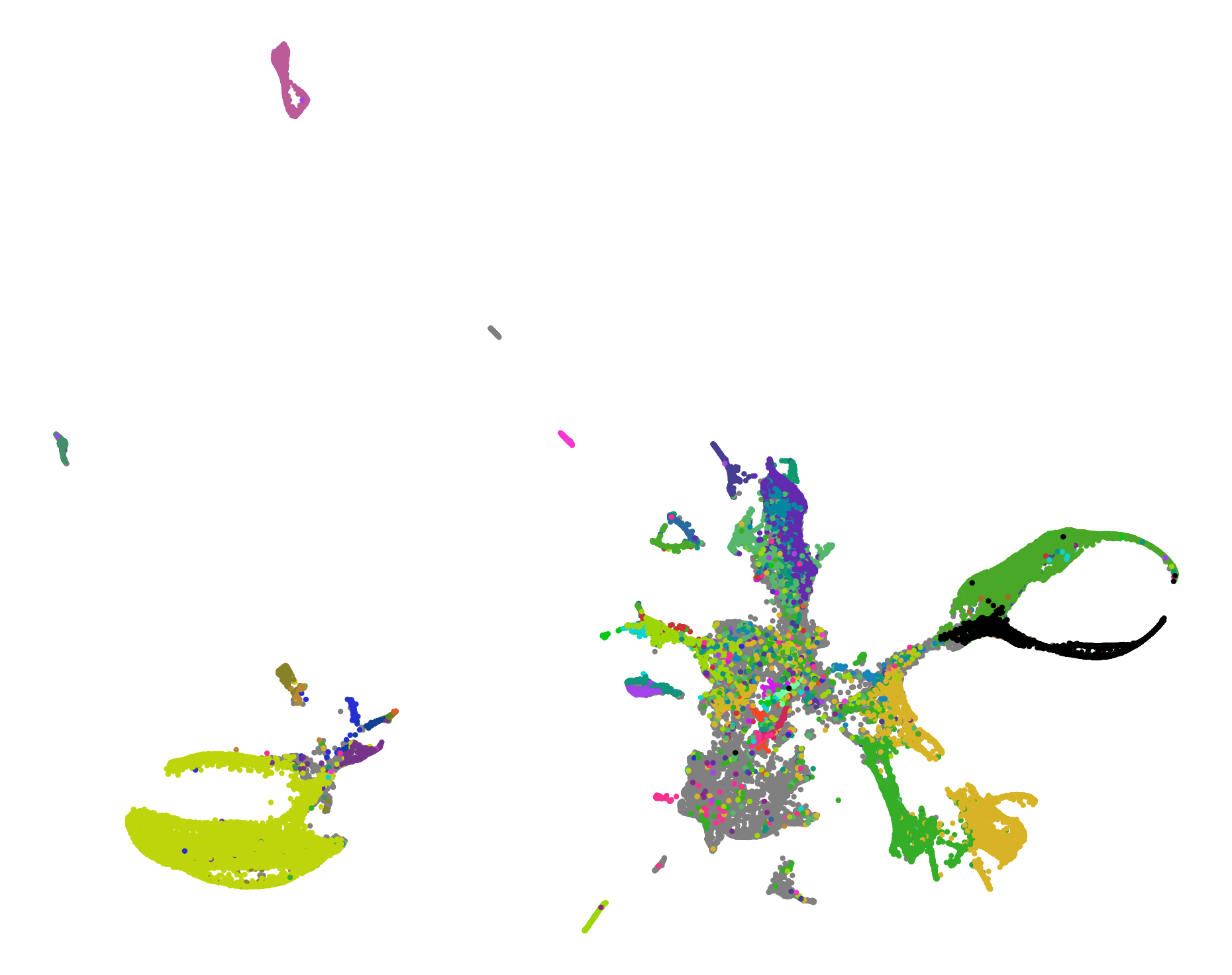}
        \caption{Shared kNN}
        \label{subfig:c_elegans_kNN_push_tail}
    \end{subfigure}
    
    \begin{subfigure}[b]{0.4\textwidth}
        \centering
        \includegraphics[width=0.9\linewidth]{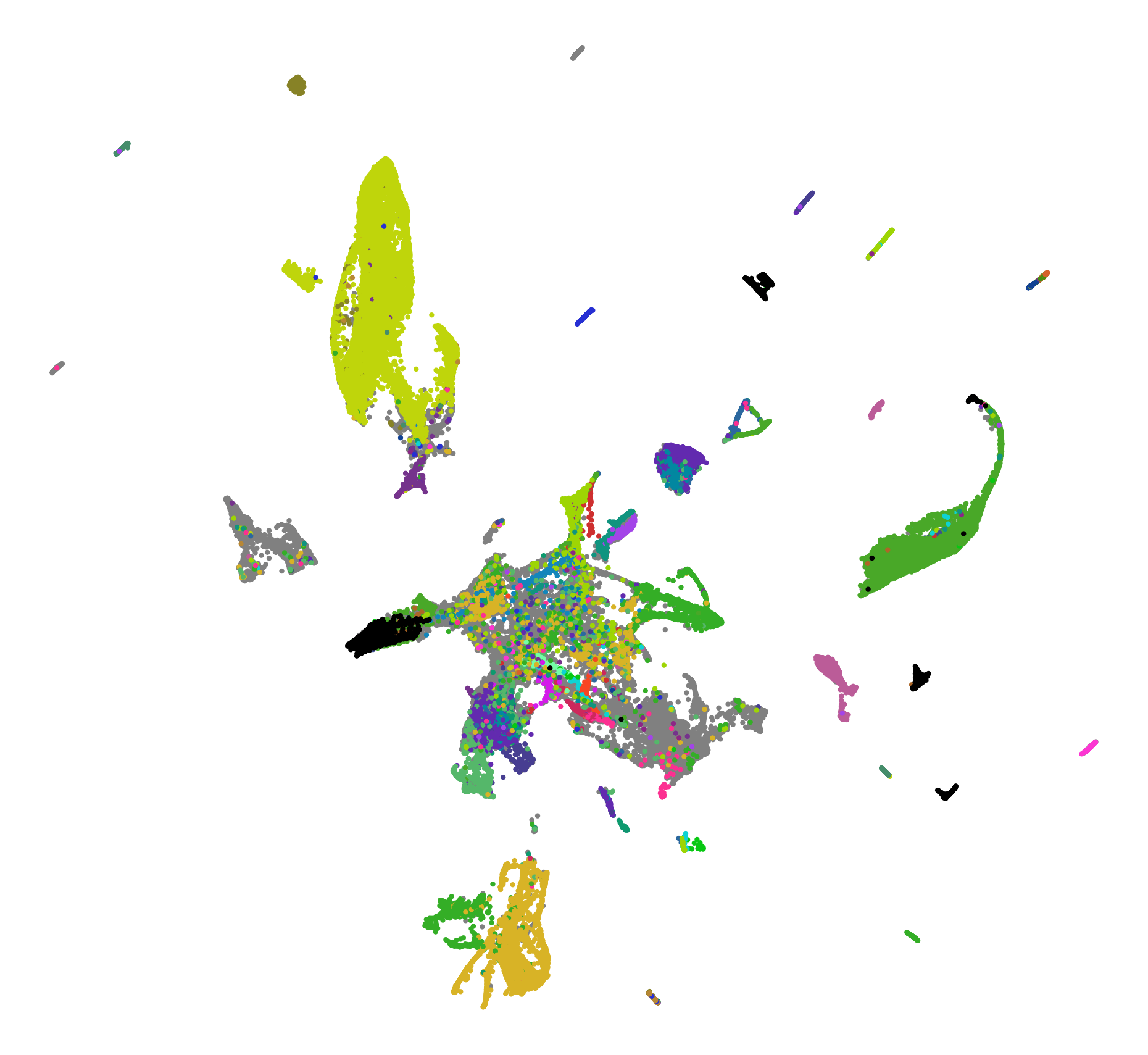}
        \caption{Permuted}
        \label{subfig:c_elegans_perm_push_tail}
    \end{subfigure}%
    \begin{subfigure}[b]{0.4\textwidth}
        \centering
        \includegraphics[width=0.9\linewidth]{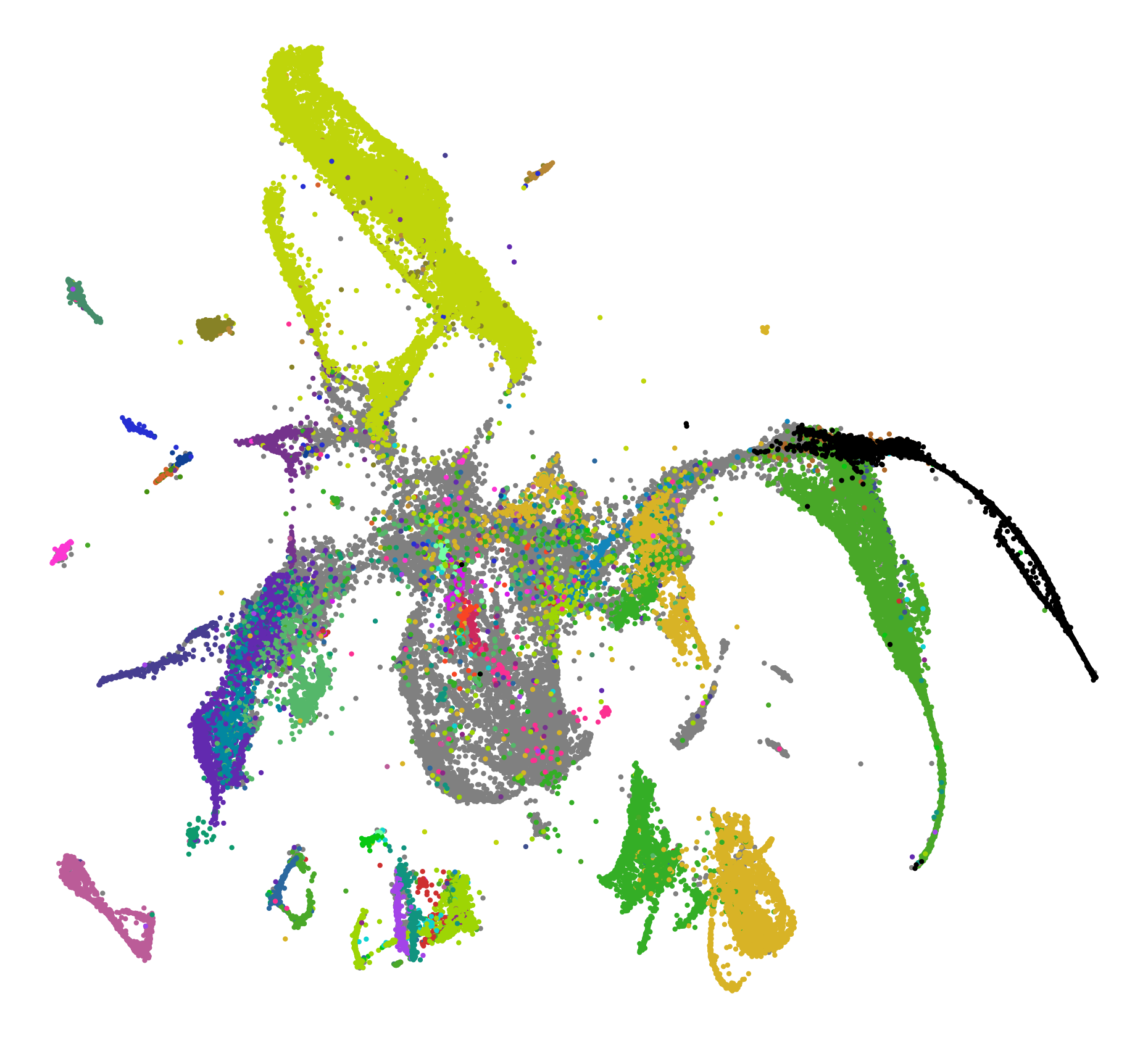}
        \caption{Inverted}
        \label{subfig:c_elegans_inv_push_tail}
    \end{subfigure}
    
    \begin{subfigure}[b]{\textwidth}
        \centering
        \includegraphics[width=0.9\linewidth]{figures/c_elegans_leg.png}
    \end{subfigure}
    \caption{Same as Figure~\ref{fig:c_elegans_perturbed} but here the tail of a negative sample is repelled from its head. There is little qualitative difference between Figure~\ref{fig:c_elegans_perturbed} and this figure overall.}
    \label{fig:c_elegans_perturbed_push_tail}
\end{figure}

\end{document}